\documentclass[accepted]{uai2023} 
 
\usepackage[british]{babel}

\usepackage{natbib} 
\usepackage{mathtools}
\usepackage{amssymb}
\usepackage{amsthm}
\usepackage{booktabs}
\usepackage{pgf,tikz}
\usepackage{mathrsfs}
\usetikzlibrary{arrows, fit, shapes, arrows.meta, positioning, matrix, decorations.pathreplacing,angles,quotes}
\usepackage{tkz-graph}
\usepackage[linesnumbered, ruled]{algorithm2e}

\newcommand{\G}{\mathcal{G}}
\newcommand{\D}{\mathcal{D}}
\newcommand{\C}{\mathcal{C}}
\newcommand{\I}{\mathcal{I}}
\newcommand{\K}{\mathcal{K}}

\newcommand{\V}{\mathbf{V}}
\newcommand{\E}{\mathbf{E}}
\newcommand{\R}{\mathcal{R}}
\newcommand{\F}{\mathcal{F}}
\newcommand{\adj}[2]{\mathrm{adj}_{{#1}}({#2})}
\newcommand{\nb}[2]{\mathrm{ne}_{{#1}}({#2})}
\newcommand{\pa}[2]{\mathrm{pa}_{{#1}}({#2})}

\newcommand{\any}{*\mkern-7mu-\mkern-7mu*}

\newtheorem{definition}{Definition}

\newtheorem{lemma}{Lemma}
\newtheorem{corollary}{Corollary}
\newtheorem{example}{Example}
\newtheorem{theorem}{Theorem}
\newtheorem{assumption}{Assumption}

\definecolor{silver}{RGB}{225, 225, 225}

\title{Do we become wiser with time? \\
On causal equivalence with tiered background knowledge}

\author[1,2]{\href{mailto:<bang@uni-bremen.de>?Subject=Your UAI 2023 paper}{Christine~W.~Bang}{}}
\author[1,2]{Vanessa~Didelez}

  \affil[1]{
    Faculty of Mathematics and Computer Science\\
    University of Bremen\\
    Bremen, Germany
}
\affil[2]{
    Leibniz Institute for Prevention Research and Epidemiology – BIPS\\
    Bremen, Germany
}
  
\begin{document}
  
\maketitle

\begin{abstract}
Equivalence classes of DAGs (represented by CPDAGs) may be too large to provide useful causal information.
Here, we address incorporating {\em tiered} background knowledge yielding restricted equivalence classes represented by `tiered MPDAGs'. Tiered knowledge leads to considerable gains in informativeness and computational efficiency: We show that construction of tiered MPDAGs only requires application of Meek's 1st rule, and that tiered MPDAGs (unlike general MPDAGs) are chain graphs with chordal components. This entails simplifications e.g. of determining valid adjustment sets for causal effect estimation. Further, we characterise when one tiered ordering is more informative than another, providing insights into useful aspects of background knowledge.
\end{abstract}

\section{Introduction}
We consider equivalence classes of DAGs represented by completed partially directed acyclic graphs (CPDAGs), occurring as outputs of causal discovery algorithms. A first characterisation of equivalent DAGs was given by \citet{verma1990} and a full characterisation of CPDAGs by \citet{andersson1997}. Often, domain expertise provides additional information about shared features of the graphs in a class. Restricting the equivalence class by background knowledge yields a type of partially directed acyclic graph (PDAG) that is potentially much more informative due to  more induced edge orientations. \citet{meek1995} provided a set of orientation rules to obtain a graph encoding the maximal implied information, and the resulting graph is then a maximally oriented partially directed acyclic graph (MPDAG) \citep{perkovic2017}. While a CPDAG represents an independence model common to all DAGs in the equivalence class, an MPDAG represents an independence model as well as additional causal or directional information that is common to all DAGs in a restricted equivalence class. DAGs and CPDAGs are special cases of MPDAGs; DAGs are MPDAGs with full (or sufficient) background knowledge, while CPDAGs are MPDAGs with no (or redundant) background knowledge. A general characterisation of MPDAGs was given by \citet{fang2022representation}. The interpretation of MPDAGs was described in  detail by \citet{perkovic2017} and is considerably more involved than that of  CPDAGs. 

Background knowledge can  be induced by, e.g., well-established causal or logical relations. Some kinds of knowledge, e.g. temporal or sequential structures, imply that the nodes can be partitioned into ordered tiers. This is the case in many settings where longitudinal data is collected, e.g. cohort or panel studies, common in sociology, epidemiology etc. In particular, this kind of data structure is used in the field of life course epidemiology \citep{kuh2004life}. Tiered background knowledge is typically unambiguous and it is intuitively obvious that it must be useful. Indeed, implementations of algorithms for constraint-based causal discovery with  a given tiered structure exist (e.g. \citet{scheines1998tetrad}), and have been applied to cohort data for life course analyses \citep{petersen2021data,foraita2022}, but the general properties of these restricted model classes have not yet been investigated. Here, we provide the first formal in-depths analysis of equivalence classes restricted by tiered knowledge. We show that there are several desirable properties, distinguishing tiered from other kinds of background knowledge. Thus, we focus on MPDAGs arising from imposing a tiered ordering, which we term `tiered MPDAGs'. We show that under the key properties of completeness and transitivity, tiered background knowledge cannot induce partially directed cycles,  and, moreover, that tiered MPDAGs are chain graphs with chordal chain components. This allows us to, e.g., apply common methods for identifying causal effects using CPDAGs  to tiered MPDAGs without any additional processing. 

While temporal structures will often be the main source for tiered background knowledge, tiers are slightly more general. Information about logical causal directions may be available, for instance between environment and individual or between cells and molecules.
When eliciting such background knowledge, there may be more effort involved to achieve more detail; or it may be possible but more costly to achieve more detail by the design of a, say, cohort study with finer waves. Existing data from cohort studies are organised in a readily available tiered structure. But within the same wave it may be possible to further subdivide the nodes using logical, temporal or similar expertise; e.g. the first wave of a children's cohort may be composed of variables  before and after birth, pertaining to mother or child etc. With view to eliciting such details or designing  a cohort study, it is therefore interesting to characterise when different tiered restrictions are redundant versus when they are most informative.

 In Section \ref{sec:background} we formalise the concepts of (tiered) background knowledge and restricted equivalence classes. We then provide a formal characterisation of tiered MPDAGs:  Section \ref{sec:properties} describes some of their properties, and Section \ref{sec:character}  compares different tiered orderings in terms of informativeness. In Section \ref{sec:sim}, we illustrate which types of tiered knowledge are particularly informative via simulation, and in Section \ref{sec:example} we provide a practical example. Section \ref{sec:comparisons} addresses how tiered information structurally differs from other types of background knowledge. Throughout, we rely on standard notation for (causal) graphical models,  an overview of  relevant definitions can be found in Section A of the Supplement; all proofs can be found in Sections D and  E of the Supplement.

\section{Background knowledge} \label{sec:background}

While the DAGs in an equivalence class have exactly the same conditional independencies, they can still have vastly different causal implications. In this section we  introduce a smaller, and possibly more informative, subclass of causal graphs using background knowledge. 

We define that \emph{background knowledge} $\K=(\R, \F)$ consists of a set of \emph{required edges} $\R$ and a set of \emph{forbidden edges} $\F$. 

\begin{definition}[Encoding background knowledge]
A graph $\G$ encodes background knowledge $\K=(\R,\F)$ if all of the edges in $\R$ and none of the edges in $\F$ are present in $\G$. 
\end{definition}

\subsection{Restricted equivalence classes}

In our setup we only consider correct background knowledge in the sense that it agrees with an underlying (unknown) true DAG:

\begin{assumption}\label{assumptionknowledge}
The given background knowledge is correct.
\end{assumption}

Let $\C$ be a CPDAG, then by $[\C]$ we denote the equivalence class of DAGs represented by $\C$.

\begin{definition}[\cite{meek1995}]
A CPDAG $\C$ and background knowledge $\K=(\R,\F)$ are \emph{consistent} if and only if there exists a DAG $\D\in [\C]$ such that all of the edges in $\R$ and none of the edges in $\F$ are in $\D$.
\end{definition} 

Since our focus is on equivalence classes, we will assume throughout:

\begin{assumption}\label{assumptionCPDAG}
The given CPDAG is correct.
\end{assumption}

Combining Assumption \ref{assumptionknowledge} and \ref{assumptionCPDAG}, background knowledge will be  consistent with the CPDAG. In actual practice, inconsistencies  might occur, e.g. due to statistical errors when first learning the CPDAG. These are separate issues which we address elsewhere.

\SetKwInOut{Input}{input} 
\SetKwInOut{Output}{output}

\begin{algorithm}[!htbp]

\caption{Constructing $\C^\K$} \label{alg:ck}

\Input{CPDAG $\C=(\V,\E)$ and consistent background knowledge $\K=(\R,\F)$.} 

\Output{PDAG $\C^\K=(\V,\E')$}

$\E'=\E$

\ForAll{$\{ V_i- V_j\}\in\E$}{

\uIf{$\{ V_i\rightarrow V_j\}\in\F$}{

replace $\{ V_i- V_j\}$ with $\{ V_i\leftarrow V_j\}$ in $\E'$

\uElseIf{$\{ V_i\rightarrow V_j\}\in\R$}{

replace $\{ V_i- V_j\}$ with $\{ V_i\rightarrow V_j\}$ in $\E'$

}
}

} 

\end{algorithm}

Consistent background knowledge is imposed on a CPDAG by orienting the corresponding undirected edges. In turn, this may allow us to  orient further undirected edges, e.g. to avoid directed cycles. \cite{meek1995} showed that maximal edge orientations implied by given background knowledge are obtained under a set of four orientation rules, also known as Meek's rules (see Figure B.1 in the Supplement).
The resulting graph then no longer represents an equivalence class, but rather a \emph{restricted equivalence class}.  

More formally, the construction is as follows: Let $\C=(\V,\E)$ be a CPDAG and let $\K=(\R,\F)$ be background knowledge consistent with $\C$. First, orient edges in $\C$ according to $\K$ as in Algorithm \ref{alg:ck}, and let $\C^\K$ denote the PDAG obtained from this procedure. Second, orient additional edges by repeated application of Meek's rules 1-4 until no further change; let $\G$ denote the resulting PDAG. Then $\G$ is the \emph{maximally oriented partially directed acyclic graph} (MPDAG) obtained from $\C$ relative to $\K$ \citep{meek1995}. 

For a given CPDAG $\C$ and background knowledge $\K$, the MPDAG obtained from $\C$ relative to $\K$ is unique. However, the origin of an MPDAG is not unique: Let $\C$ be a CPDAG, and  $\K_1$ and $\K_2$ two distinct sets of background knowledge. If repeated applications of Meek's rules to $\C^{\K_1}$ and $\C^{\K_2}$ lead to the same MPDAG, then $\K_1$ and $\K_2$ are \emph{equivalent} given $\C$. We say that $\K_2$ is \emph{redundant} relative to $\K_1$ if $\K_1\subseteq\K_2$ and $\K_1$ and $\K_2$ are equivalent. We say that a PDAG $\G_1$ is \emph{contained in} another PDAG $\G_2$ if they have the same skeleton and every directed edge in $\G_2$ is also in $\G_1$. We say that $\K_1$ is more \emph{informative} than $\K_2$ if the MPDAG relative to $\K_1$ is contained in the MPDAG relative to $\K_2$ 

\subsection{Tiered background knowledge}

In this work we focus on background knowledge about tiered structures.

\begin{definition}[Partial tiered ordering]
Let $\G$ be a PDAG with node set $\mathbf{V}$ of size $p$, and let $T\in\mathbb{N}$, $T \leq p$. A  (partial) tiered ordering of the nodes in $\mathbf{V}$ is a map $\tau: \mathbf{V}\mapsto \{ 1,\ldots ,T\}^p$ that assigns each node $V\in\mathbf{V}$ to a unique tier $t\in\{ 1,\ldots ,T\}$.
\end{definition}

A tiered ordering is partial if multiple nodes are assigned to the same tier. The following properties are implied by the definition and reflect what kind of background knowledge is encoded in a tiered ordering:
\begin{itemize}[align=left]
    \item[(Uniqueness)] A node is assigned to no more than one tier.
    \item[(Completeness)] Every node belongs to a tier.
    \item[(Transitivity)] If $\tau (A)\leq\tau (B)$ and $\tau (B)\leq\tau (C)$ then $\tau (A)\leq\tau (C)$.
\end{itemize}

A tiered ordering imposes background knowledge on a graph by demanding that no directed edges point from later tiers into  earlier tiers, i.e. specifying the forbidden edges accordingly  $\F=\{\{A\leftarrow B\}: \tau(A)<\tau(B), A,B\in\V\}$. Thus,  a tiered ordering provides information on the absence of ancestral relations, but not on their presence. Combining tiered knowledge with a PDAG, it might be possible to construct some ancestral relations which allow us to orient undirected edges as illustrated in Example \ref{smallexample}. 

\begin{example}\label{smallexample}
 Assume that we are given $\V=\{ A, B\}$ and tiered ordering $\tau$ with $\tau (A)<\tau (B)$. This corresponds to the background knowledge $\K$ with forbidden set $\F =\{A\leftarrow B\}$ and no required edges: In this case it could be possible that $A$ is a parent of $B$ or that there is no edge between them. However, if we additionally knew that $A$ and $B$ are adjacent, then our background knowledge would result in the edge orientation $\{A\rightarrow B\}$. 
\end{example}

As illustrated, the cross-tier edges play an important role and are defined as follows.

\begin{definition}[Cross-tier edge]
Let $\G=(\V,\E)$ be a PDAG and $\tau$ a tiered ordering of $\V$. An edge $\{A\rightarrow B\}\in\E$ is a cross-tier edge (relative to $\tau$) if $\tau(A)<\tau(B)$.
\end{definition}

With tiered knowledge $\tau$, all cross-tier edges of a PDAG will be directed. Conversely, $A-B$ only occurs if $\tau(A)=\tau(B)$. Since a tiered ordering $\tau$ of a node set $\V$ unambiguously implies a forbidden edge set we can refer to the MPDAG obtained from a CPDAG relative to a tiered ordering $\tau$, rather than referring to the forbidden edges implied by $\tau$. 
In view of Assumption \ref{assumptionknowledge} and \ref{assumptionCPDAG}, any tiered ordering that does not contradict the directed edges of the CPDAG will be consistent; to establish that there is no contradiction we therefore  only need to verify the cross-tier edges.

We will refer to MPDAGs relative to exclusively tiered  background knowledge as `tiered MPDAGs'. This is in contrast to general MPDAGs, which can arise from any kind of background knowledge. 

\begin{example}[Equivalence class restricted by tiered ordering] \label{example:constructTMPDAG}

\begin{figure}[!htbp]
\centering
\begin{tikzpicture}[state/.style={thick}]

\node[state] (x1) at (0,1.25) {$A$};
\node[state] (x2) at (0,0) {$B$};
\node[state] (x3) at (2.25,1.875) {$C$};
\node[state] (x4) at (2.25,0.625) {$D$};
\node[state] (x5) at (2.25,-0.625) {$E$};
\node[state] (x6) at (4.5,1.25) {$F$};
\node[state] (x7) at (4.5,0) {$G$};

\node (t0) at (0,-1.35) {$\tau = 1$};
\node (t1) at (2.25,-1.35) {$\tau = 2$};
\node (t2) at (4.5,-1.35) {$\tau = 3$};

\tikzset{dir/.style = {->, -{To[length=6, width=7]}, thick}}
\draw[dir]
(x1) edge (x2)
(x1) edge [bend left] (x3)
(x2) edge [bend right] (x5)
(x3) edge (x4)
(x4) edge (x5)
(x3) edge [bend left] (x6)
(x6) edge (x7)
;

\draw[dashed] 
(1.125,2.625) -- (1.125,-1.375)
(3.375,2.625) -- (3.375,-1.375)
;

\end{tikzpicture}
\caption{DAG $\D=(\V,\E)$ with tiered ordering $\tau$ and three tiers: $A$ and $B$ are assigned to tier 1, while $C$, $D$ and $E$ are assigned to tier 2, and $F$ and $G$ are assigned to tier 3.}
\label{fig.dag1}
\end{figure}
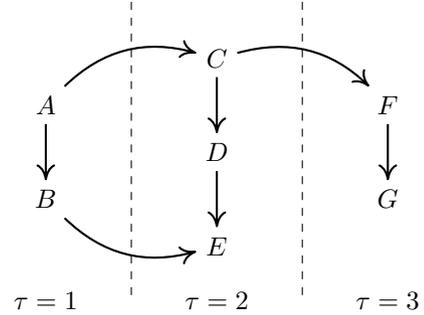

Figure \ref{fig.dag1} shows a DAG $\D$ and a tiered ordering $\tau$ of the nodes in $\D$. The differences between the equivalence class of $\D$ and the restricted equivalence class of $\D$ relative to  $\tau$ are illustrated in Figure \ref{fig.est1}. The CPDAG $\C$ represents the equivalence class of $\D$; as only two out of seven edges are directed the conditional independencies alone do not contain much causal information. Meanwhile, the restricted equivalence class represented by tiered MPDAG $\G$ relative to $\tau$  is much smaller, here six out of seven edges are oriented. $\G$ will naturally contain the same adjacencies and v-structures as $\C$,  the dashed cross-tier edges $A\rightarrow C$ and $C\rightarrow F$ are implied by the forbidden directions, and the dotted edges $C\rightarrow D$ and $F\rightarrow G$ are consequences of Meek's 1st rule which prohibits new v-structures. 
Due to these last additionally implied orientations of previously undirected edges, restricted equivalence classes given tiered background knowledge might be even smaller, and thus more informative, than one might initially expect.

\begin{figure}[!htbp]
\centering
\begin{tikzpicture}[state/.style={thick}]

\node[state] (x1') at (0,1.25) {$A$};
\node[state] (x2') at (0,0) {$B$};
\node[state] (x3') at (2.25,1.875) {$C$};
\node[state] (x4') at (2.25,0.625) {$D$};
\node[state] (x5') at (2.25,-0.625) {$E$};
\node[state] (x6') at (4.5,1.25) {$F$};
\node[state] (x7') at (4.5,0) {$G$};

\node[state] (x1) at (0,4.75) {$A$};
\node[state] (x2) at (0,3.5) {$B$};
\node[state] (x3) at (2.25,5.375) {$C$};
\node[state] (x4) at (2.25,4.125) {$D$};
\node[state] (x5) at (2.25,2.875) {$E$};
\node[state] (x6) at (4.5,4.75) {$F$};
\node[state] (x7) at (4.5,3.5) {$G$};

\node (t1) at (6, 4.5) {CPDAG $\C$};
\node (t2) at (6, 1) {MPDAG $\G$};

\tikzset{undir/.style = {-, thick}}
\tikzset{dir/.style = {->, -{To[length=6, width=7]}, thick}}
\draw[dir]
(x2) edge [bend right] (x5)
(x4) edge (x5)
(x1') edge [bend left] [dashed] (x3')
(x2') edge [bend right] (x5')
(x3') edge [dashed] [dotted] (x4')
(x4') edge (x5')
(x3') edge [bend left] [dashed] (x6')
(x6') edge [dotted] (x7')
;
\draw[undir]
(x1) edge (x2)
(x1) edge [bend left] (x3)
(x3) edge (x4)
(x3) edge [bend left] (x6)
(x6) edge (x7)
(x1') edge (x2')
;

\end{tikzpicture}
\caption{The CPDAG $\C$ representing the equivalence class of $\D$, and the tiered MPDAG $\G$ representing the restricted equivalence class of $\D$ relative to $\tau$.}
\label{fig.est1}
\end{figure}
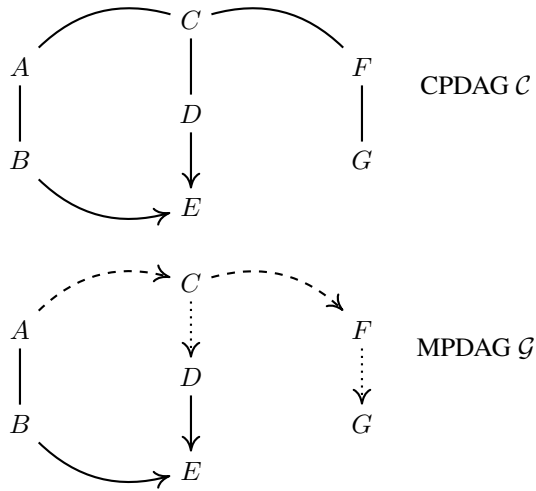

\end{example}

\section{Properties of tiered MPDAGs}\label{sec:properties}

When incorporating general background knowledge into a CPDAG, it might be necessary to apply all of Meek's rules 1-4 in order to obtain a maximally informative graph. Meek's 1st rule ensures that no new v-structures are created, while rules 2-4 all concern preventing  directed cycles. By construction,  tiered knowledge imposes an ordering of the nodes; using that this ordering is transitive and complete, the following lemma shows that Meek's 1st rule is sufficient to construct a maximally informative graph. This is a strong result, which will help us prove further results in this and the following sections.
 
\begin{lemma}\label{mainlemma}
Let $\C=(\mathbf{V},\mathbf{E})$ be a CPDAG and let $\tau$ be a tiered ordering of the nodes $\V$. Let $\C^\tau$ be the PDAG obtained according to Algorithm \ref{alg:ck} and let $\G$ be the PDAG obtained by repeatedly orienting edges in $\C^\tau$ according to Meek's rule 1 until no further change occurs. Then $\G$ is the MPDAG obtained from $\C$ relative to $\tau$. 
\end{lemma}

Note that so far we have taken a given CPDAG as the starting point to which tiered background knowledge is added. Alternatively, we can start at an earlier stage with a PDAG $\G$ that contains directed edges only if they belong to v-structures with other edges being undirected. In this case, the tiered ordering can be incorporated into $\G$ by orienting the cross-tier edges and then applying all four of Meeks rules to achieve maximality. Lemma \ref{mainlemma} therefore highlights the extra orientations implied by tiered knowledge on top of the usual orientations.

While general MPDAGs might contain partially directed cycles, it turns out that when background knowledge arises from tiered structures the completeness and transitivity of tiered orderings ensure that no partially directed cycles can occur:

\begin{theorem}\label{theorem:cycles}
Let $\G=(\V,\E)$ be a tiered MPDAG, then $\G$ does not have any partially directed cycles.
\end{theorem}

The implications of Theorem \ref{theorem:cycles}  allow us to work with tiered MPDAGs in a similar way as  with CPDAGs. The same does not hold for MPDAGs in general  \citep{perkovic2017}; we will elaborate on this in the section below.

It was shown by \cite{andersson1997} that CPDAGs are chain graphs with chordal chain components, which is useful for many purposes. Given Theorem \ref{theorem:cycles} it becomes straightforward to show that the the same holds for tiered MPDAGs:

\begin{corollary} \label{corollary:chain}
Let $\G=(\V,\E)$ be a tiered MPDAG, then $\G$ is a chain graph with chordal chain components.
\end{corollary}

Similarly, \citet{wang2022sound}  obtain graphs with chordal  undirected components under a different type of background knowledge:  
Local background knowledge for a node $A\in\V$ is defined as the knowledge of whether $A$ is a cause of $V$, for each $V\in\adj{}{A}$. Although not explicit, this induces a transitivity among the nodes in $\adj{}{A}$.

\subsection{Interpretation of undirected paths}\label{sec:interpret}

In a CPDAG $\C=(\V,\E)$, an undirected path $\pi$ between two nodes $A,B\in\V$ indicates that there exist a DAG $\D_1\in[\C]$ and another DAG $\D_2\in[\C]$  such that $A$ is an ancestor of $B$ in $\D_1$ and $B$ is an ancestor of $A$ in $\D_2$. This is not necessarily the case in a general MPDAG $\G=(\V,\E')$: If there is an edge $A\rightarrow B$ in $\G$ not on $\pi$, i.e. $\G$ has a partially directed cycle, then there exists a DAG $\D_1$ represented in the class by $\G$ in which the path corresponding to $\pi$ is directed from $A$ to $B$, but there cannot be a DAG $\D_2$ in the class represented by $\G$ in which $B$ is an ancestor of $A$ since this would create a cycle. Hence, an undirected path between two nodes in an MPDAG does not necessarily mean that the path can be directed either way. Hence it is  necessary to check multiple paths in the graph in order to determine whether one path can be directed. Since partially directed cycles do not occur in tiered MPDAGs, the above issue does not occur and, hence, the interpretation of undirected paths in tiered MPDAGs is the same as in CPDAGs. 

\subsection{Adjustment in tiered MPDAGs}\label{sec:adjust}

The \emph{generalised adjustment criterion} \citep{perkovic2018complete} determines whether a set of nodes in a CPDAG constitutes a valid adjustment set in every DAG in the equivalence class. This criterion checks for possibly causal paths, i.e. paths for which there are DAGs in the equivalence in which these paths are causal. Hence, any undirected path is possibly causal. However, as described in Section \ref{sec:interpret}, this interpretation does not hold for general MPDAGs, and the notion of possibly causal does not transfer directly from CPDAGs to MPDAGs. In order to tackle this issue, \citet{perkovic2017} introduced the notion of \emph{b-possibly causal} paths, which is a stronger requirement but ensures that there are in fact paths in the equivalence class that are causal. With this notion an adjustment criterion for general MPDAGs, the \emph{b-adjustment criterion}, can be given \citep{perkovic2017}. To determine whether a path is b-possibly causal, one needs to check multiple paths in the graph, which can be computationally heavy for large and dense graphs. For tiered MPDAGs, the definition of b-possibly causal paths simplifies to the definition of possibly causal paths, and the generalised adjustment criterion for CPDAGs is valid for tiered MPDAGs as well:

\begin{corollary}\label{corollary:possibly}
    Let $\G=(\V,\E)$ be a tiered MPDAG, and let $\mathbf{X}, \mathbf{Y},\mathbf{Z}\subseteq\V$ be pairwise disjoint node sets. Then $\mathbf{Z}$ satisfies the generalised adjustment criterion relative to $(\mathbf{X}, \mathbf{Y})$ in $\G$ if and only if it satisfies the b-adjustment criterion relative to $(\mathbf{X}, \mathbf{Y})$ in $\G$.
\end{corollary}

A related result is shown in \cite{van2016separators}: They introduce a class of graphs called \emph{restricted chain graphs}, which are chain graphs with (1) chordal chain components, and (2) no unshielded triples of the form $A\rightarrow B-C$. They give a sound and complete adjustment criterion for graphs of this type, and they provide an algorithm to find adjustment sets. Clearly, a tiered MPDAG is a type of restricted chain graph, and the results of \cite{van2016separators} hold for tiered MPDAGs.

\subsection{IDA for tiered MPDAGs}

Covariate adjustment in a CPDAG (MPDAG) requires an adjustment set to be a valid in every DAG in the (restricted) equivalence class. The IDA-algorithm \citep{maathuis2009estimating}, instead, finds an adjustment set for each DAG in the class. Enumerating all DAGs in an equivalence class is a computationally heavy task, but it can be done in polynomial time for chain graphs with chordal chain components \citep{wienobst2021polynomial}, hence also for tiered MDPAGs.

The local IDA-algorithm utilises the fact that if a valid adjustment set exists, then the parent set is always valid \citep{pearl2009causality} and considers the possible parents, i.e. all sets that are parent sets in some DAG in the equivalence class. However, as general MPDAGs can contain partially directed cycles, it cannot be verified locally whether a set of nodes is a possible parent set, and a semi-local version  was introduced to tackle this \citep{perkovic2017}. The joint IDA-algorithm  determines the joint parent sets semi-locally by orienting subgraphs \citep{nandy2017estimating}. Similar to the local IDA-algorithm, this approach fails for general MPDAGs due to potential partially directed cycles. To tackle this issue \citet{perkovic2017} introduced an additional step to check whether the oriented subgraphs are valid. In contrast, for tiered MPDAGs no additional steps are needed, and the original local and joint IDA both remain valid:

\begin{corollary}\label{corollary:ida}
    Let $\G=(\V,\E)$ be a tiered MPDAG. Let $\mathbf{PA}_{\G}(X)$ denote the multiset of parent sets of $X$ in all DAGs represented by $\G$, and let $\mathbf{PA}^{\mathrm{local}}_{\G}(X)$ denote the multiset of parent sets of $X$ obtained from the local IDA algorithm. Then $\mathbf{PA}_{\G}(X)$ and $\mathbf{PA}^{\mathrm{local}}_{\G}(X)$ contain the same distinct elements. Moreover, let $\mathbf{PA}^{\mathrm{joint}}_{\G}(X)$ denote the multiset of parent sets of $X$ obtained from the joint IDA algorithm. Then $\mathbf{PA}_{\G}(X)$ and $\mathbf{PA}^{\mathrm{joint}}_{\G}(X)$ contain the same distinct elements and the ratios of multiplicities of any two elements are the same.
\end{corollary}

In order to adapt the local IDA-algorithm to general MPDAGs, \citet{fang2020ida} introduced a set of local orientation rules to verify whether a set of nodes is a possible parent set of a given node, yielding a fully local version of the IDA that can handle general MPDAGs. While this reduces computation time for general MPDAGs,  it is not necessary for tiered MPDAGs due to Corollary \ref{corollary:ida}.

Finally, the optimal IDA-algorithm \citep{witte2020efficient}  is only semi-local and it includes an additional step to check other parts of the graph: This algorithm is essentially not local as the optimal adjustment set is not likely to be the parent set \citep{witte2020efficient, henckel2022graphical}. In this case, tiered MPDAGs do not have an advantage over general MPDAGs. However, the definition of the optimal adjustment set does simplify in a similar fashion as in Section \ref{sec:adjust}. A minimal version of the IDA-algorithm is proposed by \citet{guo2021minimal}; like the optimal IDA, this method is  non-local by construction, and tiered MPDAGs do not provide an advantage in this case.

\section{Comparing tiered background knowledge}\label{sec:character}

In this section, we investigate how different tiered background knowledge can be compared. This is relevant in situations where different experts are consulted or where eliciting more detailed knowledge may require more effort. It also provides insights into what kind of background knowledge is especially valuable. In case of a cohort study spanning a whole life-time, we have a clear tiered structure due to the time-ordering, but the tiers might be large, and we need to consult different experts in order to refine the tiers depending on their expertise in e.g. children’s health, life style factors, or diseases common among the elderly. This can be very costly and time consuming to obtain, and it is then beneficial to know how to prioritise. Moreover, time-ordering has the advantage of being correct, whereas other ways of motivating tiers might be less certain. The results of this section could therefore also be used at the design stage of a, say, cohort study: As our results will show, the finer we can reliably partition early variables into tiers by the design alone, the more informative it will be for causal structure learning.

Throughout, we will only compare tiered orderings that are compatible in the sense that they do not contradict each other on the ordering of the nodes; this is in line with Assumption \ref{assumptionknowledge}. Hence, the orderings can only disagree on the status of an edge being a cross-tier edge, not the direction of it.

Consider two different tiered orderings $\tau_i$ and $\tau_j$. If for all $A,B\in\V:$ $\tau_i (A)<\tau_i (B)\Rightarrow\tau_j (A)<\tau_j (B)$, then $\tau_j$ is \emph{finer} than $\tau_i$, and $\tau_i$ is \emph{coarser} than $\tau_j$. In this case we have for the respective sets of forbidden edges that $\F_i\subseteq\F_j$.

Note that a finer tiered ordering can be redundant compared to the coarser one; otherwise  it must be more informative. However, two tiered orderings can be different without one being finer or coarser than the other. They can then either be equivalent, or one can be more informative than the other, or they are incomparable.

To compare tiered orderings on a given CPDAG, we must compare the resulting tiered MPDAGs. We give a graphical criterion for their  equivalence in Section \ref{sec:equivalent}. Evidently, this provides a criterion for redundancy. More interestingly, this also provides insight into when one ordering is more informative than another as addressed in Section \ref{sec:compareorderings}.

\subsection{Equivalence of tiered MPDAGs}\label{sec:equivalent}

First, we need some further terminology: 

\begin{definition}[Earlier path]
Let $\G=(\V,\E)$ be a PDAG, let $\tau$ be a tiered ordering of the nodes $\V$, and let $\pi_1$ and $\pi_2$ be two arbitrary paths in $\G$. If $\pi_1$ contains a node $V$ with $\tau(V)<\tau(W)$ for all nodes $W$ on $\pi_2$, we say that $\pi_1$ is earlier than $\pi_2$; correspondingly,  $\pi_2$ is later than $\pi_1$. A path is earliest if it does not contain subpaths of any earlier paths.
\end{definition}

We say that an edge is \emph{fully shielded} if it does not occur on any unshielded path. Hence, $A\any B$ is a fully shielded edge in $\G$ iff $\adj{\G}{A}\backslash\{ B\}=\adj{\G}{B}\backslash\{ A\}$.

In the following, $\C_u^\tau$  refers to the graph obtained by first omitting every directed edge in a graph $\C$, and then orienting edges in $\C_u$ according to the tiered ordering $\tau$. 

\begin{theorem}\label{mainthm}
Let $\C$ be a CPDAG and let $\tau_1$ and $\tau_2$ be distinct tiered orderings. Let $\G_1$ be the tiered MPDAG obtained from $\C$ relative to $\tau_1$ and $\G_2$ the tiered MPDAG obtained from $\C$ relative to $\tau_2$. Then $\G_1=\G_2$ if and only if

\begin{itemize}
    \item[(i)] $\C^{\tau_1}_u$ and $\C^{\tau_2}_u$ agree on the first cross-tier edge on any earliest unshielded path and
    \item[(ii)] $\C^{\tau_1}_u$ and $\C^{\tau_2}_u$ agree on any fully shielded cross-tier edge.
\end{itemize}
\end{theorem}

A path can have at most two first cross-tier edges. If a path contains two first cross-tier edges, the above condition requires $\C^{\tau_1}_u$ and $\C^{\tau_2}_u$ to agree on both cross-tier edges.

\begin{example}[Equivalent tiered orderings]\label{example:mainthm}

We again consider the DAG $\D$ and the tiered ordering $\tau$ in Figure \ref{fig.dag1} from Example \ref{example:constructTMPDAG}. Now we compare $\tau$ to the tiered ordering $\tau'$ with
\begin{align*}
    \tau'(A)&=\tau'(B)=1\\
    \tau'(C)&=\tau'(D)=\tau'(E)=\tau'(F)=\tau'(G)=2
\end{align*}

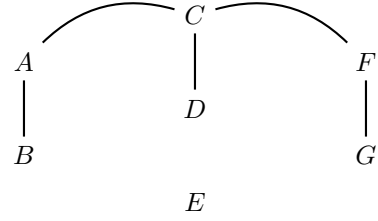
\begin{figure}[!htbp]
\centering
\begin{tikzpicture}[state/.style={thick}]

\node[state] (x1) at (0,1.25) {$A$};
\node[state] (x2) at (0,0) {$B$};
\node[state] (x3) at (2.25,1.875) {$C$};
\node[state] (x4) at (2.25,0.625) {$D$};
\node[state] (x5) at (2.25,-0.625) {$E$};
\node[state] (x6) at (4.5,1.25) {$F$};
\node[state] (x7) at (4.5,0) {$G$};

\tikzset{undir/.style = {-, thick}}
\draw[undir]
(x2) edge (x1)
(x1) edge [bend left] (x3)
(x3) edge (x4)
(x3) edge [bend left] (x6)
(x6) edge (x7)
;

\end{tikzpicture}
\caption{The undirected subgraph $\C_u$ of the CPDAG $\C$ from Figure \ref{fig.est1}.}
\label{fig:cu}
\end{figure}

In Figure \ref{fig:cu} we have the undirected subgraph $\C_u$ of the CPDAG $\C$ of $\D$.   There are two earliest unshielded paths in $\C_u$: $B-A-C-F-G$ and $B-A-C-D$. Clearly, $\C_u^{\tau}$ and $\C_u^{\tau'}$ agree on the first cross-tier edges on these paths: On $\langle B, A, C, F, G\rangle $ the first cross-tier edge is $A\rightarrow C$, and on $\langle B, A, C, D\rangle$ the first cross-tier edge is $A\rightarrow C$. Note that in these graphs, no edge is fully shielded. Orienting edges in $\C_u^{\tau}$ and $\C_u^{\tau'}$ repeatedly according to Meek's 1st rule results in the same MPDAG. Since $\tau$ is finer than $\tau'$, it follows that $\tau$ is redundant relative to $\tau'$ given $\C$. 
\end{example}

\subsection{Comparing tiered orderings}\label{sec:compareorderings}

We have now characterised when two tiered orderings produce the same MPDAG; as a consequence of this, Theorem \ref{mainthm} also gives rise to useful insights concerning when tiered knowledge is more informative. 

\begin{corollary}\label{corollary:informative}
Let $\C=(\V, \E)$ be a CPDAG  and let $\tau_1$ and $\tau_2$ be two distinct tiered orderings of $\V$. Assume that 
\begin{itemize}
    \item[(i)] every first cross-tier edge on an earliest unshielded path in $\C^{\tau_2}_u$ is also a cross-tier edge in $\C^{\tau_1}_u$, and 
    \item[(ii)] every fully shielded cross-tier edge in $\C^{\tau_2}_u$ is also a cross-tier edge in $\C^{\tau_1}_u$,
\end{itemize}
then $\tau_1$ is more informative than $\tau_2$ given $\C$ if 
\begin{itemize}
    \item[(iii)] $\C^{\tau_1}_u$ has a first cross-tier edge on an earliest unshielded path, that is not a cross-tier edge in $\C^{\tau_2}_u$, or
    \item[(iv)] $\C^{\tau_1}_u$ has more fully shielded cross-tier edges than $\C^{\tau_2}_u$.
\end{itemize}
\end{corollary}

In Corollary \ref{corollary:informative}, (i) and (ii) ensure that the MPDAG $\G_1$ relative to $\tau_1$ 
is contained in the MPDAG $\G_2$ relative to $\tau_2$. Additional information can be obtained in two ways: Either directly from the tiered ordering, reflected in condition (iv), or indirectly through Meek's 1st rule, reflected in condition (iii).

Consider first condition (iii). This condition does not mention the number of cross-tier edges, it is only important to know the earliest.

\begin{example}\label{example:unshielded}
Consider again the DAG $\D=(\V,\E)$ from Example \ref{example:constructTMPDAG} and Example \ref{example:mainthm}. Consider now the tiered ordering $\tau_1$ with 
\begin{align*}
\tau_1(A)&=\tau_1(B)=1\qquad\qquad\qquad\qquad\qquad \\ 
\tau_1(C)&=2\\
\tau_1(D)&=\tau_1(E)=3\\
\tau_1(F)&=4 \textit{, and}\\ 
\tau_1(G)&=5
\end{align*}

Clearly, $\tau_1$ results in the same MPDAG as $\tau$ and $\tau'$; since $\tau_1$ is finer than both $\tau$ and $\tau'$ it is redundant given $\C$. Consider now the tiered ordering $\tau_2$ of $\V$ with 
\begin{align*}
\tau_2(A)&=\tau_2(B)=\tau_2(C)=1\qquad\qquad\qquad \\ 
\tau_2(D)&=\tau_2(E)=2\\
\tau_2(F)&=3 \textit{, and}\\ 
\tau_2(G)&=4
\end{align*}

Let $\G'$ be the MPDAG relative to $\tau_2$ given $\C$. Then $\G'$ has the same edge orientations as the MPDAG relative to $\tau$, $\tau'$ and $\tau_1$, except for $\{A\rightarrow C\}$ which remains undirected in $\G'$. Here, $\tau_1$ is finer than $\tau_2$, but $\tau_1$ is not redundant. In fact, $\tau_1$ is more informative than $\tau_2$ due to condition (iii) of Corollary \ref{corollary:informative}. In addition, $\tau$ and $\tau'$ are both more informative than $\tau_2$, even though $\tau_2$ assigns the nodes to more tiers. 
\end{example}

Consider condition (iv) in Corollary \ref{corollary:informative}. This suggests that every fully shielded cross-tier edge provides unique information. Hence, there is an immediate gain in information from each additional fully shielded cross-tier edge in $\C_u$, since edges of this type can not be oriented by Meek's 1st rule. Hence, in the complete subgraphs of $\C_u$, no tiered background knowledge is redundant. 

\begin{example}[Fully shielded edges]
Consider the simple case of a CPDAG $\C=(\V,\E)$ with three nodes $\V=\{ A, B, C\}$, where $\C$ is complete: $\E=\{\{A-B\},\{B-C\},\{A-C\}\}$. Assume that the true ordering $\tau_\alpha$ assigns the nodes to individual tiers with $\tau_\alpha(A)<\tau_\alpha(B)<\tau_\alpha(C)$. There are then three types of partial orderings that are compatible with $\tau_\alpha$: Orderings that assigns the nodes to the same tier, e.g. an ordering $\tau_\beta$ with  $\tau_\beta(A)=\tau_\beta(B)=\tau_\beta(C)$, or orderings that assigns two nodes to the same tier and the third to an individual tier, e.g. $\tau_\gamma$ and $\tau_\delta$ with $\tau_\gamma(A)<\tau_\gamma(B)=\tau_\gamma(C)$ and $\tau_\delta(A)=\tau_\delta(B)<\tau_\delta(C)$.

\begin{figure}[!htbp]
\centering
\begin{tikzpicture}[state/.style={thick}]

\node[state] (A1) at (0,0) {$A$};
\node[state] (B1) at (1.5,0) {$B$};
\node[state] (C1) at (3,0) {$C$};

\node[state] (A2) at (4,0) {$A$};
\node[state] (B2) at (5.5,0) {$B$};
\node[state] (C2) at (7,0) {$C$};

\node[state] (A3) at (0,-2.75) {$A$};
\node[state] (B3) at (1.5,-2.75) {$B$};
\node[state] (C3) at (3,-2.75) {$C$};

\node[state] (A4) at (4,-2.75) {$A$};
\node[state] (B4) at (5.5,-2.75) {$B$};
\node[state] (C4) at (7,-2.75) {$C$};

\node (t1) at (1.5, -0.75) {$\G_\alpha$};
\node (t2) at (5.5, -0.75) {$\G_\beta$};
\node (t3) at (1.5, -3.5) {$\G_\gamma$};
\node (t4) at (5.5, -3.5) {$\G_\delta$};

\tikzset{undir/.style = {-, thick}}
\tikzset{dir/.style = {->, -{To[length=6, width=7]}, thick}}
\draw[dir]
(A1) edge (B1)
(A1) edge [bend left=60] (C1)
(B1) edge (C1)

(A3) edge (B3)
(A3) edge [bend left=60] (C3)

(A4) edge [bend left=60] (C4)
(B4) edge (C4)
;
\draw[undir]
(A2) edge (B2)
(A2) edge [bend left=60] (C2)
(B2) edge (C2)

(B3) edge (C3)

(A4) edge (B4)
;

\end{tikzpicture}
\caption{Top left: MPDAG $\G_\alpha$ relative to $\tau_\alpha$; note that this is equal to the underlying DAG. Top right: MPDAG $\G_\beta$ relative to $\tau_\beta$; note that this is equal to the CPDAG $\C$. Bottom left: MPDAG $\G_\gamma$ relative to $\tau_\gamma$. Bottom right: MPDAG $\G_\delta$ relative to $\tau_\delta$.}
\label{fig:ex3}
\end{figure}
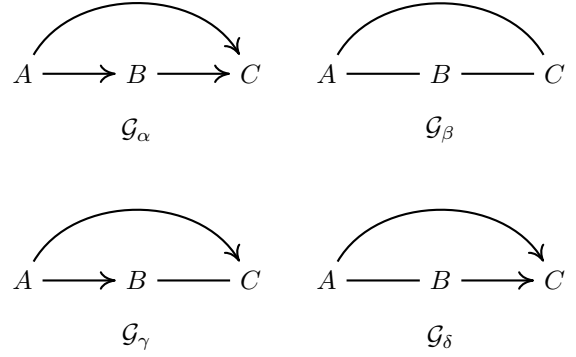

Figure \ref{fig:ex3} shows that the MPDAGs $\G_\alpha$ (relative to $\tau_\alpha$), $\G_\beta$ (relative to $\tau_\beta$), $\G_\gamma$ (relative to $\tau_\gamma$), and $\G_\delta$ (relative to $\tau_\delta$) are distinct. All oriented edges are implied by the tiered background knowledge and no edge has been oriented as a consequence of Meek's 1st rule. Here, $\tau_\delta$ and $\tau_\gamma$ are incomparable, and they are both more informative than $\tau_\beta$. Moreover, $\tau_\alpha$ is more informative than the other orderings.

\end{example}

\subsection{Simulation study}\label{sec:sim}

Corollary \ref{corollary:informative} shows that in graphs with many unshielded paths we can potentially obtain a large amount of additional information, and the  earlier we are able to identify the direction of a causal path, the more information we can gain. In summary: (1) early knowledge is in general more beneficial than late knowledge, even if the late knowledge is more detailed (c.f. Example \ref{example:unshielded}), and (2) since we expect unshielded paths to occur more frequently in sparse graphs, the effect of Meek's 1st rule is expected to be more pronounced in sparse than in dense graphs. In order to investigate to which degree these two features occur in practice, we conducted a simulation study. 

\begin{figure}[!htbp]
    \centering
    \includegraphics{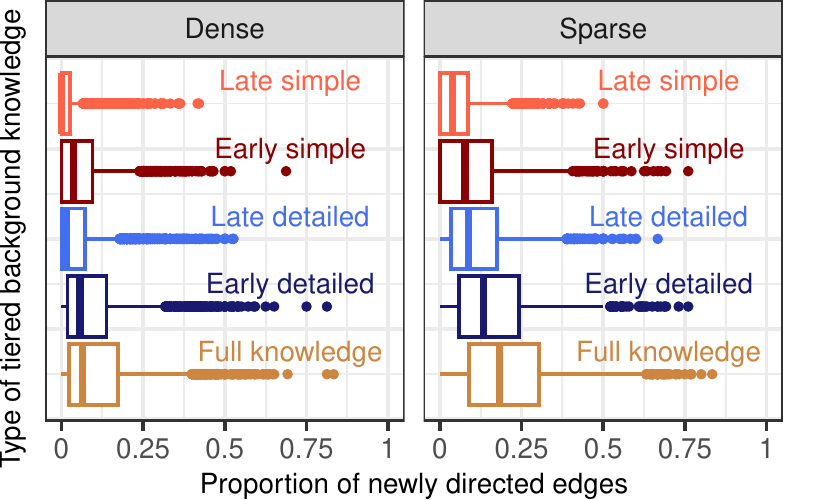}
    \caption{Results of one setting of the simulation. 6000 random DAGs with 25 nodes were generated; half of them sparse, the other half dense. For each DAG and tiered ordering, the tiered MPDAG was constructed and the difference in number of directed edges to its corresponding CPDAG was computed, divided by the total number of edges.}
    \label{fig:sim_main}
\end{figure}

To adhere to Assumption \ref{assumptionknowledge}, we generated random DAGs and for each random DAG, we considered five different, consistent tiered orderings: full knowledge, early detailed knowledge, late detailed knowledge, early simple knowledge and late simple knowledge. For each DAG, we constructed its CPDAG, and for each combination of DAG and tiered ordering, we constructed the tiered MPDAG. To adhere to Assumption \ref{assumptionCPDAG}, the CPDAGs and MPDAGs were constructed based on the independence models encoded by the DAGs; hence, finite sample issues did not occur. For each MPDAG, we counted the difference in number of directed edges between the MPDAG and its corresponding CPDAG, and divided this by the total number of edges; this  measures the fraction of edges that cannot be oriented in the CPDAG, but can be oriented in the tiered MPDAG. Since we compare oracle CPDAGs to oracle MPDAGs, this  measures exactly the (relative) gain in informativeness. A detailed description of the study can be found in Section C in the Supplement.

 Figure \ref{fig:sim_main} shows that using the full knowledge of the orderings unsurprisingly provides most new orientations. Moreover, we see that early knowledge is more beneficial than late, even though they are equally detailed, which is in line with (1) above. Interestingly, in dense graphs early knowledge can also induce more oriented edges than late knowledge, even if the later knowledge is more detailed, which is also in line with (1). Additionally, the advantage of adding tiered background knowledge is relatively larger in sparse graphs than dense graphs, which is in line with (2). 
 
 We performed the same procedure for random DAGs with node sets of size 10, 50 and 100, for which we found analogous results. These results are  depicted in Figure C.2 in the Supplement.

 \subsection{Practical example}\label{sec:example}

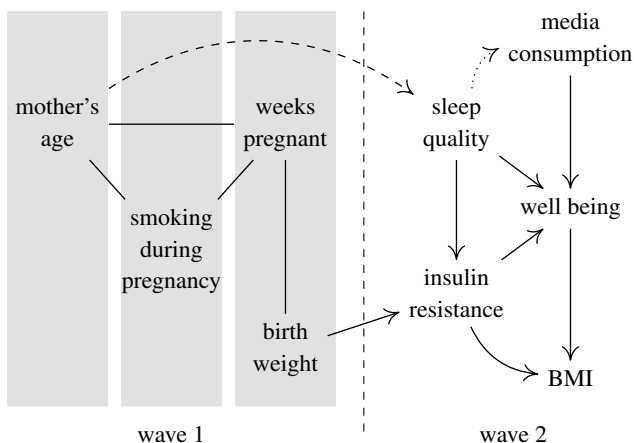
\begin{figure}[!htbp]
\centering
\begin{tikzpicture}[state/.style={thick}, scale=0.75]

\filldraw[silver] (-.875,4.5) rectangle +(1.75,-7) ;

\filldraw[silver] (1.125,4.5) rectangle +(1.75,-7) ;

\filldraw[silver] (3.125,4.5) rectangle +(1.75,-7) ;

\node[align=center] (A) at (0,2.5) {\footnotesize mother's\\ \footnotesize age};
\node[align=center] (B) at (2,0.25) {\footnotesize smoking\\ \footnotesize during\\ \footnotesize pregnancy};
\node[align=center] (C) at (4,2.5) {\footnotesize weeks\\ \footnotesize pregnant};
\node[align=center] (D) at (4,-1.5) {\footnotesize birth\\ \footnotesize weight};

\node[align=center] (E) at (7,2.5) {\footnotesize sleep\\ \footnotesize quality};
\node[align=center] (F) at (7,-.5) {\footnotesize insulin\\ \footnotesize resistance};

\node[align=center] (G) at (9,4) {\footnotesize media\\ \footnotesize consumption};
\node (H) at (9,1) {\footnotesize well being};
\node (I) at (9,-2) {\footnotesize BMI};

\tikzset{dir/.style = {->, -{To[length=5.5, width=6.5]}, line width = 0.5pt}}
\draw[dir]
(A) edge [bend left=35] [dashed] (E)

(D) edge (F)

(E) edge (F)
(E) edge [bend left] [dotted] (G)
(E) edge (H)

(F) edge  (H)
(F) edge [bend right] (I)

(G) edge (H)

(H) edge (I)

;

\draw[dashed] 
(5.375,4.5) -- (5.375,-3)
;

\tikzset{undir/.style = {-, line width = 0.5pt}}
\draw[undir]
(A) edge (B)
(A) edge (C)

(B) edge (C)

(C) edge (D)
;

\node (t1) at (2,-3) {\footnotesize wave 1};
\node (t2) at (8,-3) {\footnotesize wave 2};

\end{tikzpicture}
\caption{Simplified example of a cohort study. Early life factors are measured at wave 1, childhood health factors  at wave 2.  Edges are oriented by the v-structures, time-ordering and Meek's rules. Expert knowledge  allows the first tier to be subdivided into three new tiers.}
\label{fig:practicalexample}
\end{figure}

Figure \ref{fig:practicalexample} is a simplified example based on cohort data analysed in \citet{foraita2022}, it shows the MPDAG obtained from the time-ordering of wave 1 and 2. The corresponding CPDAG only has two fewer directed edges (the dashed and dotted), so the time-ordering does not provide much new information. The induced subgraph over the first wave remains undirected. In order to obtain a more informative graph, we should consult early life experts rather than children's life style experts; alternatively, the cohort study should have been designed such that early life factors were measured at different time points. While mother's age naturally is determined before smoking during pregnancy, which again occurs before birth, experts could disagree on the causal order of pregnancy duration and birth weight, which are defined at the exact same time. However, it is not necessary to order these particular two nodes, here: Any ordering $\tau$ with $\tau$(mothers age)$<\tau$(smoking during pregnancy)$<\tau$(remaining nodes) allows for the {\em entire} graph to be oriented in this case.

\section{Relation to other work} \label{sec:comparisons}

Subject matter background knowledge can come from different sources and take different forms, and previous work provides results for other knowledge than tiered one. A distinct type of causal background knowledge is, for instance, obtainable when experimentation is possible, see \cite{hauser2012characterization} for a characterisation of interventional equivalence classes of DAGs. While a tiered ordering is given before learning a graph, experiments can be performed iteratively, and the choice of most informative interventions might depend on the given intermediate graphical structure. It has been shown in \cite{eberhardt2008almost} and \cite{hauser2014two} that the most informative strategy is to intervene on nodes in the largest undirected complete subgraphs: this yields most new edge orientations, including those following from  Meek's rules. Since the knowledge obtained from an intervention is local, this type of background knowledge lacks the completeness of tiered knowledge. This means that all four of Meek's rules might apply after orienting edges according to interventional knowledge, resulting in additional orientations within a complete subgraphs, in contrast to tiered background knowledge. 

\cite{mooij2020joint} considered context variables. These can be seen as a special case of tiered knowledge: Some of the variables, the context variables, form an earlier tier, and others, the system variables, form later tiers. However, there can be additional knowledge about presence/absence of relations between context variables, or their causal relations may not be of interest: In these cases we are no longer in the tiered framework. 

Background knowledge about non-ancestral (pairwise) relations,  considered by \cite{fang2020ida},  can be seen as a non-complete version of tiered knowledge. They show that knowledge of non-ancestral relations can be translated to a set of direct causal relations, i.e. directed edges. In contrast, the completeness and transitivity of tiered knowledge subsumes such relations through the orientations of cross-tier edges. A more general representation of background knowledge is provided in \cite{fang2022representation}, where   ancestral background knowledge on node pairs is considered; this is surprisingly different from tiers and cannot necessarily be encoded  graphically. The authors provide a criterion for checking equivalence of background knowledge, which is more general than the one provided here since tiered background knowledge can be considered a complete version of pairwise causal constraints. However, unlike \citet {fang2022representation}, our criterion can be checked on the graph, and due to the properties of tiered orderings, it is rather simple. 

Multivariate time series, as repeated measurements of the same variables over time \citep{malinsky2018causal, runge2019inferring}, have a very obvious and unambiguous tiered ordering, and in this sense our results extend to time series. But because time series are observations on a single unit over a long time instead of multiple i.i.d. observations, the models typically impose additional  structure. For instance, that each variable depends on its own past, thus forcing edges, which is a restriction that we have not considered; and that there is a limited memory (e.g. $k$-th order), thus disallowing edges, which we have also not considered; and importantly, for time series a stationarity or slow/smooth change of the structure is enforced which is also not covered by tiered background knowledge. 
Our results on the ensuing equivalence classes, such as absence of partially directed cycles, still apply under these additional structural assumptions as they restrict the  skeleton  of the true CPDAG, and the tiered (in this case temporal) background knowledge complements them.  
Considering informativeness, sometimes it may be possible to impose additional tiers within time-slices if there is a known order for the contemporaneous variables, e.g. due to biological processes underlying medical time series; this could be useful in obtaining more edge orientations. 

A different line of work on using background knowledge relaxes the assumption of causal sufficiency. Latent variables can be accommodated in maximal ancestral graphs (MAGs) (\cite{richardson2002ancestral}), and the corresponding equivalence classes are represented by partial ancestral graphs (PAGs), see characterization by \cite{ali2009markov}. Different orientation rules are needed, and a set of ten rules were introduced by \cite{zhang2008} ensuring the maximally informative PAG. For added background knowledge it has not yet been shown, in general, that these ten rules yield a maximally informative graph. However, this has been shown for tiered background knowledge under the extra assumptions of no cross-tier confounding and no selection bias \citep{andrews2020completeness}. Moreover, in this case it turns out that not all ten orientation rules are needed,  similarly to our Lemma \ref{mainlemma}. We therefore conjecture that analogous results to those of section \ref{sec:properties} extend to PAGs with tiered background knowledge under the assumptions of no selection bias and no cross-tier confounding. Further work is still needed to relax the often implausible assumption of no cross-tier confounding.

\section{Discussion}

By formalising equivalence classes restricted by tiered orderings, we provided some new insights: Tiered MPDAGs do not have partially directed cycles and are chain graphs with chordal chain components; this makes them easier to handle and interpret, e.g. for causal effect estimation. We have given a characterisation of tiered MPDAGs which clarified what can be gained by adding tiered knowledge and what will still remain unknown. Sparse graphs, in particular, will benefit much from edge orientations implied by the tiered ordering; further, eliciting background knowledge to separate out early tiers is especially informative. Hence, this is when we do become `wiser with time'.

In summary, we believe that oftentimes background knowledge comes in the form of a tiered ordering. Moreover, tiered knowledge can be expected to  be reliable especially when based on temporal information. A benefit of the tiered orderings is that in case of doubts or disagreements about  the ordering, these may be resolved by  coarsening the tiers and thus arrive at, say, a consensus among differing expert opinions. Tiered MPDAGs are, therefore, at least as plausible as their corresponding CPDAGs without background knowledge. In addition, tiered MPDAGs will also be at least as informative as their corresponding CPDAGs -- in practice they will often be much more informative. While it is self-evident that any background knowledge should be exploited for causal structure learning, we have illustrated how specific aspects and how much the background knowledge results in an information gain that goes well beyond the orientation of cross-tier edges. 

\begin{acknowledgements} 
This project was funded by the Deutsche Forschungsgemeinschaft (DFG, German Research Foundation) – Project 281474342/GRK2224/2. Support from DFG (SFB 1320 EASE) is also acknowledged. We would like to thank the reviewers for their helpful and constructive comments. 
\end{acknowledgements}

\bibliography{references}

\clearpage

\appendix

\section*{Supplement}

\counterwithin{definition}{section}

\renewcommand{\thedefinition}{\Alph{section}.\arabic{definition}}

\counterwithin{figure}{section}

\renewcommand{\thefigure}{\Alph{section}.\arabic{figure}}

\counterwithin{corollary}{section}

\renewcommand{\thecorollary}{\Alph{section}.\arabic{corollary}}

\counterwithin{lemma}{section}

\renewcommand{\thelemma}{\Alph{section}.\arabic{lemma}}

\counterwithin{algocf}{section}

\renewcommand{\thealgocf}{\Alph{section}.\arabic{algocf}}

\section{Terminology}

\paragraph{Nodes and edges.} We define a \emph{graph} $\G = (\V, \E)$ as a collection of \emph{nodes} (or \emph{vertices}) $\V$ and \emph{edges} $\E$.  Edges can be either \emph{undirected} $A-B$ or \emph{directed} $A \rightarrow B$. By $A\any B$ we denote an arbitrary edge, i.e. this serves as a placeholder for either a directed or undirected edge. Two nodes $A,B\in\V$ are \emph{adjacent} in $\G$ if $\{ A\any B\}\in\E$. No node can be adjacent to itself, and there can be at most one edge between any pair of nodes. We say that an edge of the form $A\rightarrow B$ is directed out of $A$ (into $B$), and we then say that $A$ is a \emph{parent} of $B$. If there is an undirected edge between two nodes $A-B$, we say that $A$ and $B$ are \emph{neighbours}. Let $A\in\V$ be a node in a graph $\G=(\V,\E)$, then $\nb{\G}{A}$/$\adj{\G}{A}$/$\pa{\G}{A}$ is the set of neighbours/adjacent nodes/parents of $A$ in $\G$. A graph is \emph{complete} if all its nodes are adjacent. The \emph{skeleton} of a graph is the undirected graph obtained by replacing its directed edges with undirected edges. 

\paragraph{Subgraphs.} We call $\G'=(\V',\E')$ a \emph{subgraph} of $\G=(\V,\E)$ if $\V'\subseteq\V$ and $\E'\subseteq\E$. By $\G_u=(\V, \E_u)$ we denote the \emph{undirected subgraph} of $\G$, where $\E_u$ is obtained from $\E$ by removing all directed edges. Correspondingly, $\G_d=(\V, \E_d)$ is the \emph{directed subgraph} of $\G$, where $\E_d$ is obtained from $\E$ by removing all undirected edges. Let $\mathbf{A}\subseteq\V$, then the \emph{induced subgraph} of $\G$ over $\mathbf{A}$ is $\G_\mathbf{A}=(\mathbf{A},\E_\mathbf{A})$ where $\E_\mathbf{A}\subseteq\E$ contains all the edges between the nodes in $\mathbf{A}$. 

\paragraph{Paths and cycles.} A \emph{path} $\pi=\langle V_1,V_2,...,V_{K-1}, V_K \rangle$ from $V_1\in\V$ to $V_K\in\V$ of length $K$ consists of a sequence of distinct nodes where $V_i\in\adj{}{V_{i+1}}$ for $1\leq i<K$. A path from a set $\mathbf{A}\subseteq\V$ to another set $\mathbf{B}\subseteq\V$ is a path from some $A\in\mathbf{A}$ to some $B\in\mathbf{B}$. The \emph{subpath} of $\pi$ from $V_i$ to $V_j$ for $1\leq i\leq j\leq K$ is  $\pi(V_i,V_j)=\langle V_i, V_{i+1},\ldots ,V_{j-1},V_j\rangle$. Let $\G'=(\V,\E')$ be a graph with same skeleton as $\G=(\V,\E)$ but possibly $\E'\neq\E$, then for a path $\pi$ in $\G$ its \emph{corresponding path} in $\G'$ is the path $\pi'$ in $\G'$ consisting of the same nodes as $\pi$. An \emph{undirected path} consists only of undirected edges. A \emph{directed path} from $V_1$ to $V_K$ has all edges oriented towards $V_K$, i.e. $V_{j}\rightarrow V_{j+1}$ for all $1\leq j<K$; then $V_1$ is an \emph{ancestor} of $V_K$ ($V_K$ is a \emph{descendant} of $V_1$). A path from $V_1$ to $V_K$ that contains both directed and undirected edges with at least one edge $V_{j}\rightarrow V_{j+1}$ for some $1\leq j<K$ directed towards $B$ and no edge $V_{j}\leftarrow V_{j+1}$ for any $1\leq j<K$ is a \emph{partially directed path} from $V_1$ to $V_K$. An undirected (directed) path from $V_1$ to $V_K$ combined with an undirected (directed) path from $V_K$ to $V_1$ we call an \emph{undirected (directed) cycle}. An undirected or partially directed path from $V_1$ to $V_K$ combined with a directed or partially directed path from $V_K$ to $V_1$ we call a \emph{partially directed cycle}. 

\paragraph{(Partially) directed acyclic graphs.}  A graph consisting of only undirected edges is an \emph{undirected graph}. An undirected graph is \emph{chordal} if every cycle of length $\geq 4$ has an adjacent pair of non-consecutive nodes. A \emph{directed acyclic graph} (DAG) is  a graph containing only directed edges and no directed cycles. A partially directed acyclic graph (PDAG) is a graph containing both directed and undirected edges and no directed cycles; DAGs and undirected graphs are special cases of PDAGs. A \emph{chain graph} is a PDAG that does not have any partially directed cycles. The \emph{chain components} of a chain graph are the undirected subgraphs.

\paragraph{Colliders, (un-) shielded and v-structures.} We call a triple $\langle A, B, C\rangle$ \emph{unshielded} if $A$ and $B$ are adjacent, $B$ and $C$ are adjacent, and $A$ and $C$ are not adjacent. We call a path unshielded if all triples on the path are unshielded. If a triple of the form $A\rightarrow B\leftarrow C$ occurs, we call $B$ a \emph{collider}, and if the triple is unshielded we call it a \emph{v-structure}.

\paragraph{d-separation}

\begin{definition}[d-connecting]
Let $\pi$ be a path in some PDAG $\G=(\V,\E)$, and let $\mathbf{C}\subset\mathbf{V}$. If (i) every collider $V$ on $\pi$, or a descendant of $V$, is in $\mathbf{C}$, and (ii) no non-collider on $\pi$ is in $\mathbf{C}$, then $\pi$  is d-connecting given $\mathbf{C}$.
\end{definition}

If there exists a path from a set of nodes $\mathbf{A}$ to another set of nodes $\mathbf{B}$, where $\mathbf{A}\cap \mathbf{B}=\emptyset$, that is d-connecting given $\mathbf{C}$, we say that $\mathbf{A}$ and $\mathbf{B}$ are \emph{d-connected} given \textbf{C}. If no such path exists, we say that $\mathbf{A}$ and $\mathbf{B}$ are \emph{d-separated} given \textbf{C}, and we denote this by
\begin{align*}
\mathbf{A}\perp_d \mathbf{B}\mid \mathbf{C}
\end{align*}
We define an \emph{independence model} $\I (\G )$ induced by a graph $\G$ as the collection of all $d$-separations in $\G$: 
\begin{align*}
(\mathbf{A}\perp_d  \mathbf{B}\mid \mathbf{C})\in\I (\G) \Leftrightarrow \text{$A$ and $B$ are d-sep. by $\mathbf{C}$ in $\G$}
\end{align*}

\paragraph{Markov equivalence and CPDAGs.} We say that two graphs $\G_1$ and $\G_2$ are \emph{Markov equivalent} if they induce the same independence model: $\I(\G_1)=\I(\G_2)$; an \emph{equivalence class} is  a class of Markov equivalent graphs. A \emph{completed partially directed acyclic graph} (CPDAG) represents an equivalence class of DAGs, and can consist of undirected as well as directed edges: Undirected edges represent edges for which there exists at least one DAG in the equivalence class where the edge is oriented in one direction, and at least one DAG, where it is oriented in the opposite direction. Directed edges represent edges that must be identical in every DAG contained in the equivalence class. Two DAGs belong to the same equivalence class if and only if they have the same skeleton and the same v-structures \citep{verma1990}. A graph is \emph{maximally informative} if no additional edge can be oriented without restricting the equivalence class.  A \emph{restricted equivalence class} is a class of Markov equivalent graphs, that encode some additional common information. A \emph{maximally oriented partially directed acyclic graph} (MPDAG) represents a restricted equivalence class.

\section{Previous results}

\subsection{Meek's rules}

An equivalence class of DAGs is uniquely characterised by the skeleton and v-structures \citep{verma1990}, but more directed edges might be shared among the DAGs in the class. \cite{meek1995} introduced a set of four orientation rules (Figure \ref{fig.meeksrules}), often referred to as \emph{Meek's rules}, for which the graphical output will be maximally informative. Given the correct skeleton and v-structures of some equivalence class, repeated application of rules 1-3 outputs a CPDAG. Given the correct skeleton and v-structures, and additional background knowledge, repeated application of rules 1-4 outputs an MPDAG.

\label{sec:meek}
\begin{figure}[!htbp]
\centering
\begin{tikzpicture}[state/.style={thick}]

\node (r1) at (1,-0.5) {Rule 1};
\node (a1) at (1,-1) {$\Longrightarrow$};
\node (r2) at (1,-3.5) {Rule 2};
\node (a2) at (1,-4) {$\Longrightarrow$};
\node (r3) at (1,-6.5) {Rule 3};
\node (a3) at (1,-7) {$\Longrightarrow$};
\node (r4) at (1,-9.5) {Rule 4};
\node (a4) at (1,-10) {$\Longrightarrow$};

\node (i) at (-1.125,-2.25) {(i)};
\node (i') at (2.875,-2.25) {(i')};
\node (ii) at (-1.125,-5.25) {(ii)};
\node (ii') at (2.875,-5.25) {(ii')};
\node (iii) at (-1.125,-8.25) {(iii)};
\node (iii') at (2.875,-8.25) {(iii')};
\node (iv) at (-1.125,-11.25) {(iv)};
\node (iv') at (2.875,-11.25) {(iv')};

\node[state] (A1) at (-2,0) {$A$};
\node[state] (B1) at (-0.25,0) {$B$};
\node[state] (C1) at (-0.25,-1.75) {$C$};

\node[state] (A1') at (2,0) {$A$};
\node[state] (B1') at (3.75,0) {$B$};
\node[state] (C1') at (3.75,-1.75) {$C$};

\node[state] (A2) at (-2,-3) {$A$};
\node[state] (B2) at (-0.25,-3) {$B$};
\node[state] (C2) at (-0.25,-4.75) {$C$};

\node[state] (A2') at (2,-3) {$A$};
\node[state] (B2') at (3.75,-3) {$B$};
\node[state] (C2') at (3.75,-4.75) {$C$};

\node[state] (A3) at (-2,-6) {$A$};
\node[state] (B3) at (-0.25,-6) {$B$};
\node[state] (C3) at (-2,-7.75) {$C$};
\node[state] (D3) at (-0.25,-7.75) {$D$};

\node[state] (A3') at (2,-6) {$A$};
\node[state] (B3') at (3.75,-6) {$B$};
\node[state] (C3') at (2,-7.75) {$C$};
\node[state] (D3') at (3.75,-7.75) {$D$};

\node[state] (A4) at (-2,-9) {$A$};
\node[state] (B4) at (-0.25,-9) {$B$};
\node[state] (C4) at (-2,-10.75) {$C$};
\node[state] (D4) at (-0.25,-10.75) {$D$};

\node[state] (A4') at (2,-9) {$A$};
\node[state] (B4') at (3.75,-9) {$B$};
\node[state] (C4') at (2,-10.75) {$C$};
\node[state] (D4') at (3.75,-10.75) {$D$};

\tikzset{dir/.style = {->, -{To[length=6, width=7]}, thick}}
\draw[dir]
(A1) edge (B1)
(A1') edge (B1')
(B1') edge [dashed] (C1')
(A2) edge (B2)
(B2) edge (C2)
(A2') edge (B2')
(B2') edge (C2')
(A2') edge [dashed] (C2')
(B3) edge (D3)
(C3) edge (D3)
(B3') edge (D3')
(C3') edge (D3')
(A3') edge [dashed] (D3')
(A4) edge (B4)
(B4) edge (D4)
(A4') edge (B4')
(B4') edge (D4')
(C4') edge [dashed] (D4')
; 
\tikzset{undir/.style = {-,  thick}}
\draw[undir]
(B1) edge (C1)
(A2) edge (C2)
(A3) edge (B3)
(A3) edge (C3)
(A3) edge (D3)
(A3') edge (B3')
(A3') edge (C3')
(A4) edge (C4)
(B4) edge (C4)
(C4) edge (D4)
(A4') edge (C4')
(B4') edge (C4')
;                    
        
\end{tikzpicture}

\caption{Meek's rules. If (i), (ii), (iii) or (iv) occur as an induced subgraph of some PDAG, then orient them as (i'), (ii'), (iii') or (iv'), respectively.}
\label{fig.meeksrules}
\end{figure}
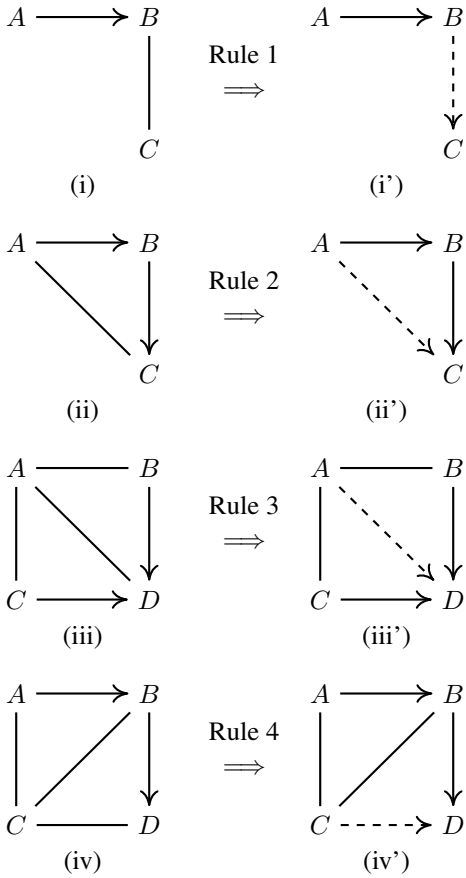

\subsection{Adjustment criterion}

In a CPDAG $\C=(\V,\E)$, a path $\pi=\langle V_1,\ldots ,V_K\rangle$ is \emph{possibly causal} from $V_1$ to $V_K$ if it does not contain an edge $V_i\leftarrow V_{i+1}$ with $1\leq i<K$. Otherwise it is \emph{non-causal} from $V_1$ to $V_K$.

\begin{definition}[b-possibly causal \citep{perkovic2017}]
Let $\G=(\V,\E)$ be an MPDAG and let $\pi=\langle V_1,\ldots ,V_K\rangle$ be a path in $\G$. Then $\pi$ is \emph{b-possibly causal} from $V_1$ to $V_K$ in $\G$ if and only if no edge $V_i \leftarrow V_j$, $1 \leq i < j \leq K$ is in $\G$. Otherwise, $\pi$ is \emph{b-non-causal path} in $\G$.
\end{definition}

\subsection{IDA-algorithm}

Let $\G=(\V,\E)$ be a graph, and let $\mathbf{X}\subseteq\V$. Then we denote the set of parents of $\mathbf{X}$ in $\G$ by $\pa{\G}{\mathbf{X}}=\underset{X\in\mathbf{X}}{\cup}\pa{\G}{X}$.

Let $\G=(\V,\E)$ be an MPDAG, let $X\in\V$ and let $\mathbf{S}\subseteq\nb{\G}{X}$. Then $\G_{\mathbf{S}\rightarrow X}$ is the PDAG obtained by orienting all undirected edges $Z-X$ to $Z\rightarrow X$ if $Z\in\mathbf{S}$ and $Z\leftarrow X$ if $Z\in\nb{\G}{X}\backslash\mathbf{S}$. A set of nodes $\mathbf{P}\subseteq\V$ is a \emph{valid} (\emph{jointly valid}) parent set of $X$ ($\mathbf{X}$) if there exists a DAG $\D$ in the class represented by $\G$ for which $\pa{\D}{X}=\mathbf{P}$ ($\pa{\D}{\mathbf{X}}=\mathbf{P}$).

\SetKwInOut{Input}{input} 
\SetKwInOut{Output}{output}

\begin{algorithm}[!htbp]

\label{alg:ida_local}

\caption{Locally obtaining valid parent sets from a tiered MPDAG using local IDA \citep{maathuis2009estimating}}

\Input{Tiered MPDAG $\G=(\V,\E)$, node $X\in\V$} 

\Output{Multiset $\mathbf{PA}_{\G}^{\mathrm{local}}(X)$}

\vspace{5pt} 

$\mathbf{PA}_{\G}^{\mathrm{local}}(X)=\emptyset$\\

\ForAll{$\mathbf{S}\subseteq\nb{\G}{X}$}{

\uIf{$\G_{\mathbf{S}\rightarrow X}$ has no new v-structure with $X$ as collider}{
add $\pa{\G}{X}\cup\mathbf{S}$ to $\mathbf{PA}_{\G}^{\mathrm{local}}(X)$

} 

} 

\end{algorithm}

\begin{algorithm}[!htbp]

\label{alg:ida_joint}

\caption{Semi-locally obtaining jointly valid parent sets from a tiered MPDAG using joint IDA \citep{nandy2017estimating}}

\Input{Tiered MPDAG $\G=(\V,\E)$, set of nodes $\mathbf{X}\subseteq\V$, $\mathbf{X}=\{X_1,\ldots , X_k\}$} 

\Output{Multiset $\mathbf{PA}_{\G}^{\mathrm{joint}}(\mathbf{X})$}

\vspace{5pt} 

Obtain $\G_u$ and $\G_d$ from $\G$\\

Obtain the connected components of $\G_u$ that contain at least one node of $\mathbf{X}$: $\G_{u,1},\ldots ,\G_{u,l}$ for $l\leq k$

\For{$i=1,\ldots ,l$}{

Let $\mathbf{PA}_i$ be the multiset of all jointly valid parent sets of the nodes of $\mathbf{X}$ in $\G_{u,i}$ obtained by constructing all DAGs in the (restricted) equivalence class represented by $\G_{u,i}$.

} 

Construct $\mathbf{PA}_u$ by taking all possible combinations of $\mathbf{PA}_1,\ldots ,\mathbf{PA}_l$\\

$\mathbf{PA}_{\G}^{\mathrm{joint}}(\mathbf{X})=\{\mathbf{PA}_1'\cup\pa{G_d}{X_1},\ldots ,\mathbf{PA}_k'\cup\pa{\G_d}{X_k}\}$\\ where $(\mathbf{PA}_1',\ldots ,\mathbf{PA}_k')\in\mathbf{PA}_u$.

\end{algorithm}

\section{Simulation study}

Simulations were done in \textsf{R} version 4.2.1 using the \texttt{pcalg} package version 2.7-8, and random DAGs were simulated using the \texttt{randDAG} function. We simulated 8 different types of DAGs: The DAGs had either 10, 25, 50 or 100 nodes, and the structure was either dense or sparse. Sparse graphs had an expected number of adjacent nodes of 2, while dense graphs had an expected number of adjacent nodes of 5. Each DAG type was simulated three times, using either the Erdös-Rényi method, power-law method or geometric method.

We assumed that the full tiered ordering of the nodes assigned them to 5 tiers of equal size; hence, the tier size was either 2, 5, 10 or 20 depending on the number of nodes in the graph. We compared the full knowledge of the five tiers to four combinations of early or late, and more or less detailed knowledge. An overview of the tiered orderings can be found in Figure \ref{fig:orderings}. For each DAG, we constructed its CPDAG, and for each combination of DAG and tiered ordering $\tau_{\mathrm{full}}$ (full knowledge), $\tau_{\mathrm{early1}}$ (early simple), $\tau_{\mathrm{early2}}$ (early detailed), $\tau_{\mathrm{late1}}$ (late detailed) or $\tau_{\mathrm{late2}}$ (late detailed), we constructed the tiered MPDAG. For each MPDAG, the number of additional directed edges compared to its corresponding CPDAG was counted. The above was repeated 1000 times for each combination of DAG type and simulation method; i.e. a total of 24,000 simulations. 

The differences between the tiered MPDAGs and the corresponding CPDAGs are visualised in the boxplots in Figure \ref{fig:sim_supplement1} and in Figure \ref{fig:sim_main} in the main text. In Figure \ref{fig:sim_supplement1} and Figure \ref{fig:sim_main} we consider the number of new directed edges divided by the total number of edges in the graphs; the raw numbers are depicted in Figure \ref{fig:sim_supplement2}.

\begin{figure}[!htbp]

\centering
\resizebox{225pt}{110pt}{
\begin{tikzpicture}

\node (t1) at (0,0) {$\tau_{\mathrm{full}}=0$};
\node (t2) at (2,0) {$\tau_{\mathrm{full}}=1$};
\node (t3) at (4,0) {$\tau_{\mathrm{full}}=2$};
\node (t4) at (6,0) {$\tau_{\mathrm{full}}=3$};
\node (t5) at (8,0) {$\tau_{\mathrm{full}}=4$};

\draw[dashed] 
(1,0.75) -- (1,-0.75)
(3,0.75) -- (3,-0.75)
(5,0.75) -- (5,-0.75)
(7,0.75) -- (7,-0.75)
;

\draw[decoration={brace,mirror,raise=5pt},decorate]
(-1,-0.75) -- node[below=6pt] {$\tau_{\mathrm{early2}}=0$} (0.95,-0.75);

\draw[decoration={brace,mirror,raise=5pt},decorate]
(1.05,-0.75) -- node[below=6pt] {$\tau_{\mathrm{early2}}=1$} (2.95,-0.75);

\draw[decoration={brace,mirror,raise=5pt},decorate]
(3.05,-0.75) -- node[below=6pt] {$\tau_{\mathrm{early2}}=2$} (9,-0.75);

\draw[decoration={brace,mirror,raise=5pt},decorate]
(-1,-1.75) -- node[below=6pt] {$\tau_{\mathrm{late2}}=0$} (4.95,-1.75);

\draw[decoration={brace,mirror,raise=5pt},decorate]
(5.05,-1.75) -- node[below=6pt] {$\tau_{\mathrm{late2}}=1$} (6.95,-1.75);

\draw[decoration={brace,mirror,raise=5pt},decorate]
(7.05,-1.75) -- node[below=6pt] {$\tau_{\mathrm{late2}}=2$} (9,-1.75);

\draw[decoration={brace,mirror,raise=5pt},decorate]
(1.05,-2.75) -- node[below=6pt] {$\tau_{\mathrm{early1}}=0$} (9,-2.75);

\draw[decoration={brace,mirror,raise=5pt},decorate]
(-1,-2.75) -- node[below=6pt] {$\tau_{\mathrm{early1}}=1$} (0.95,-2.75);

\draw[decoration={brace,mirror,raise=5pt},decorate]
(-1,-3.75) -- node[below=6pt] {$\tau_{\mathrm{late1}}=0$} (6.95,-3.75);

\draw[decoration={brace,mirror,raise=5pt},decorate]
(7.05,-3.75) -- node[below=6pt] {$\tau_{\mathrm{late1}}=1$} (9,-3.75);

\end{tikzpicture}
}
\caption{Overview of the tiered orderings used for the simulation study. The tiered ordering $\tau_{\mathrm{full}}$ is the full ordering of the nodes. The orderings $\tau_{\mathrm{early1}}$ and  $\tau_{\mathrm{late1}}$ assign the nodes to two tiers: The main difference between these two is that $\tau_{\mathrm{early1}}$ is able to distinguish the earliest tier, while $\tau_{\mathrm{early2}}$ is able to distinguish the latest tier. The tiered orderings $\tau_{\mathrm{early2}}$ and  $\tau_{\mathrm{late2}}$ assign the nodes to three tiers: While $\tau_{\mathrm{early2}}$ contains knowledge of early tiers, $\tau_{\mathrm{late2}}$ contains knowledge of later tiers. }
\label{fig:orderings}
\end{figure}
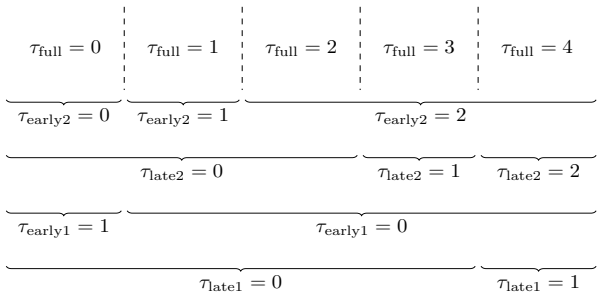

\begin{figure}[!htbp]
    \centering
    \includegraphics{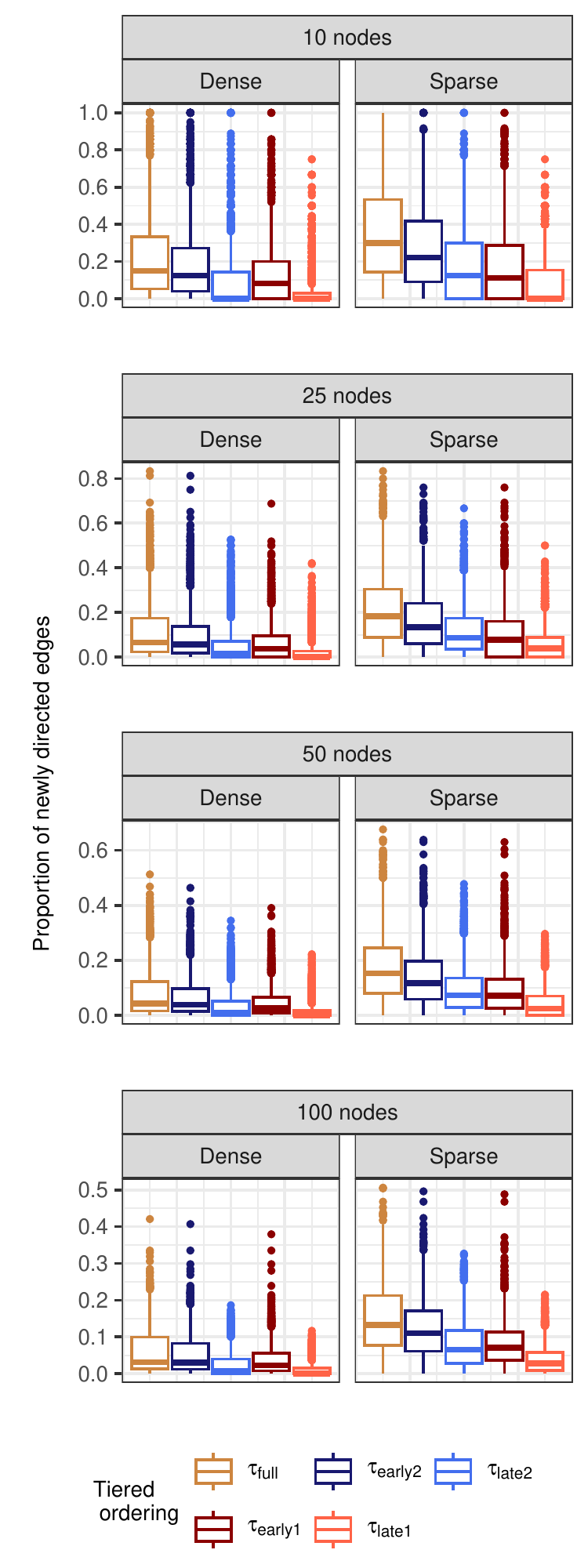}
    \caption{Results of the simulation study. 24,000 random DAGs with 10, 25, 50 or 100 nodes were generated; half of them sparse, the other half dense. For each random DAG and each tiered ordering, the tiered MPDAG was constructed and the difference in number of directed edges to its corresponding CPDAG was computed and divided by the total number of edges.}
    \label{fig:sim_supplement1}
\end{figure}

\begin{figure}[!htbp]
    \centering
    \includegraphics{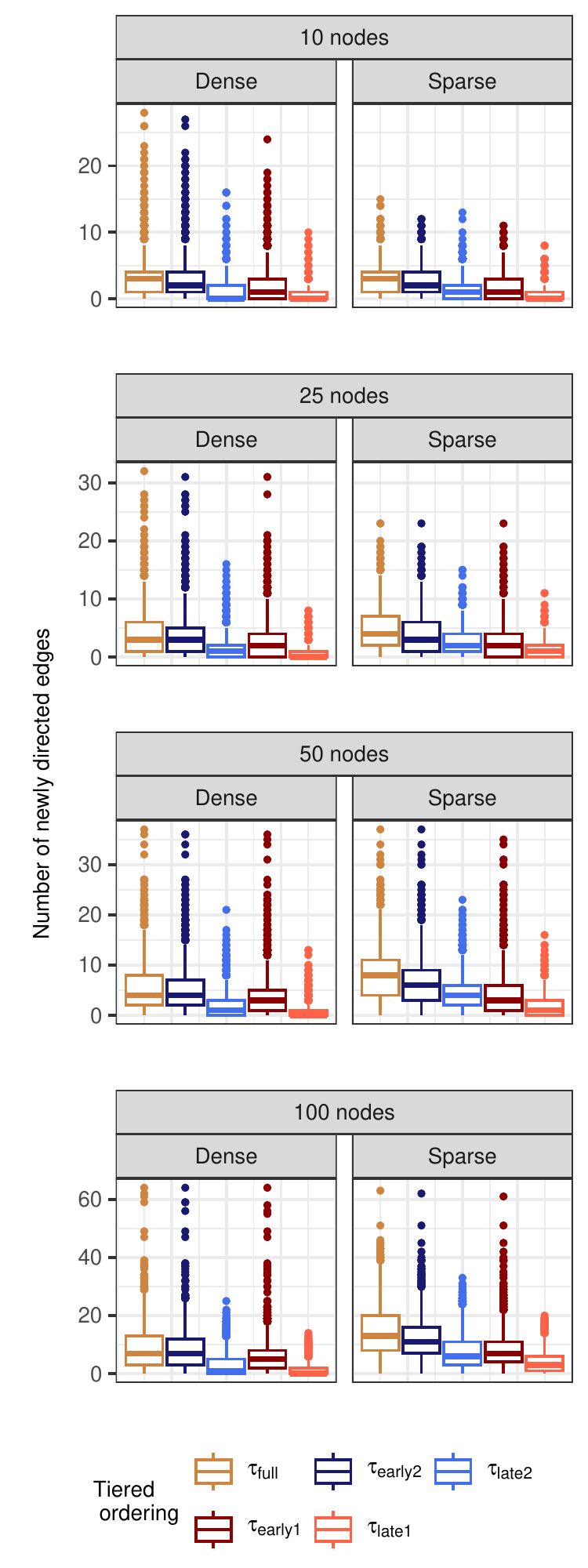}
    \caption{Results of the simulation study. 24,000 random DAGs with 10, 25, 50 or 100 nodes were generated; half of them sparse, the other half dense. For each random DAG and each tiered ordering, the tiered MPDAG was constructed and the difference in number of directed edges to its corresponding CPDAG was computed.}
    \label{fig:sim_supplement2}
\end{figure}

\section{Proofs for section \ref{sec:properties}}

\subsection{Proof of Lemma \ref{mainlemma}}

\begin{proof}
Since the MPDAG is unambiguously defined by the equivalence class and tiered ordering, if $\G$ is an MPDAG, then by construction it is the MPDAG of $\C$ relative to $\tau$. Hence, we need to show that it is in fact an MPDAG.

We proceed in two steps: (1) We show that in $\C^\tau$ an induced subgraph like in Figure \ref{fig.meeksrules} (i) can occur, while no induced subgraphs like in Figures \ref{fig.meeksrules} (ii)-(iv) can occur. (2) Let $\C^{\tau,n}$ be the graph obtained by applying Meek's 1st rule to $\C^{\tau}$ $n$ times. We will show that Figure \ref{fig.meeksrules} (i) can occur as an induced subgraph of $\C^{\tau, n}$, while Figures \ref{fig.meeksrules} (ii)-(iv) cannot occur as induced subgraphs of $\C^{\tau, n}$. This means that the resulting graph $\G$ is maximally informative, and we conclude that it is an MPDAG.

(1) \emph{Rule 1:} Assume that there is an undirected induced subgraph of $\C$ over $\{A,B,C\}\subseteq\mathbf{V}$ with adjacencies as in Figure \ref{fig.meeksrules} (i). We can obtain a triple with orientations identical to Figure \ref{fig.meeksrules} (i) in $\C^\tau$ if we have $A-B-C$ in $\C$ and $\tau (A)<\tau (B)=\tau (C)$; then $\tau$ will force $A\rightarrow B$.

\emph{Rule 2:} Assume that there is an induced subgraph of $\C$ over $\{A,B,C\}\subseteq\mathbf{V}$ with adjacencies as in Figure \ref{fig.meeksrules} (ii). Since $\C$ does not contain any partially directed cycles, this subgraph will have either 3, 2 or 0 directed edges. The case with 3 directed edges is not relevant, as well as any orientation of $\langle A, B, C\rangle$ other than $A\rightarrow B\rightarrow C$; however, the latter cannot occur in $\C$ since $\C$ is maximally informative. Hence, only an undirected subgraph in $\C$ allows for a subgraph like Figure \ref{fig.meeksrules} (ii) in $\C^\tau$. If there are edges $A\rightarrow B$ and $B\rightarrow C$ in $\C^{\tau}$ they must have been forced by $\tau$ through $\tau (A) <\tau (B)<\tau (C)$. By transitivity this implies $\tau (A) <\tau (C)$, and $A\rightarrow C$ will be oriented by $\tau$ as well.

\emph{Rule 3:} If \ref{fig.meeksrules} (iii) is an induced subgraph of $\C^{\tau}$, then it is also an induced subgraph of $\C$, since the v-structure cannot be newly forced by $\tau$. However, \ref{fig.meeksrules} (iii) cannot be an induced subgraph of $\C$ since $\C$ is maximally informative. 

\emph{Rule 4:} Assume that there is an induced subgraph of $\C$ over $\{A,B,C,D\}\subseteq\mathbf{V}$ with adjacencies as in Figure \ref{fig.meeksrules} (iv). For the case to be non-trivial, we exclude any subgraphs with other directed edges than $A\rightarrow B$ and $B\rightarrow D$; since $\C$ does not have any partially directed cycles, the subgraph must be undirected. If $A\rightarrow B\rightarrow D$ occurs in $\C^\tau$ it must be forced by $\tau$ through $\tau (A)<\tau (B)<\tau (D)$. Either $\tau (A)<\tau (C)$, $\tau (A)=\tau (C)$, or $\tau (A) >\tau (C)$. If $\tau (A)<\tau (C)$ or  $\tau (A) >\tau (C)$ then it follows that $A\rightarrow C$ or $A\leftarrow C$ according to $\tau$. If $\tau (A)=\tau (C)$, then by transitivity $\tau (C)<\tau (B)<\tau (D)$, and we orient $B\leftarrow C\rightarrow D$ according to $\tau$. 

(2) \emph{Rule 1:} Assume that there is an undirected induced subgraph of $\C$ over $\{A,B,C\}\subseteq\mathbf{V}$ with adjacencies as in Figure \ref{fig.meeksrules} (i). Assume that there is an undirected, unshielded path $\langle V_1,\ldots V_K=A\rangle$ of length $K> 1$ in $\C$ with $V_{K-1}\notin\adj{\C}{B}$.  Assume that $\tau (V_1)<\tau (V_2)$ such that $V_1\rightarrow V_2$ in $\C^\tau$ and assume that $n\geq K-1$: then $V_1\rightarrow\ldots\rightarrow A\rightarrow B$ in $\C^{\tau, n}$, and we obtain \ref{fig.meeksrules} (i).

\emph{Rule 2:} Assume that there is an induced subgraph of $\C$ over  $\{A,B,C\}\subseteq\mathbf{V}$  with adjacencies as in Figure \ref{fig.meeksrules} (ii). By the same argument as above, only an undirected induced subgraph of $\C$ can lead to an induced subgraph like \ref{fig.meeksrules} (ii) in $\C^{\tau, n}$. Moreover, by the argument above, we know that Figure \ref{fig.meeksrules} (ii) does not occur as an induced subgraph of $\C^{\tau}$; hence, we consider the case where $\tau(A)=\tau(B)=\tau(C)$ and this subgraph is undirected. The only way that $A\rightarrow B$ can be directed in $\C^{\tau,n}$ and not in $\C^\tau$ is if there is an undirected unshielded path $\langle V_1,\ldots ,V_K=A\rangle$ in $\C$ of length $K>1$ in $\C$ with $V_{K-1}\notin\adj{\C}{B}$ where $\tau (V_1)<\tau (V_2)=\tau (V_3)=\ldots=\tau (A)$ and $n\geq K-1$ such that $V_1\rightarrow\ldots\rightarrow A\rightarrow B$ in $\C^{\tau,n}$. In order for $A-C$ to remain undirected in $\C^{\tau,n}$, it must be the case that $V_{K-1}\in\adj{\C}{C}$. If $V_{K-2}\notin\adj{\C}{C}$ then $V_{K-1}\rightarrow C- B$ and $C\rightarrow B$ will be directed by Meek's 1st rule; hence, assume $V_{K-2}\in\adj{\C}{C}$. Assume now that $V_j\in\adj{\C}{C}$ for some $1\leq j\leq K-2$. Either (a) $V_{j-1}\notin\adj{\C}{C}$ or (b) $V_{j-1}\in\adj{\C}{C}$. (a) If $V_{j-1}\notin\adj{\C}{C}$ then $V_j\rightarrow C- B$ occurs and it must then be the case that $V_j\in\adj{\C}{B}$ in order for $C-B$ not to be directed as $C\rightarrow B$ or create a new v-structure, such that $B\rightarrow C$ would have been in $\C$. We then have $A\rightarrow B \any V_j$: this cannot be a v-structure since then $A\rightarrow B$ would have been oriented in $\C$ and if $B\rightarrow V_j$ we would have had cycle; hence $V_j\in\adj{\C}{A}$. Then $V_j\in\adj{\C}{V_{K-1}}$ since otherwise $V_{K-1}\rightarrow A\any V_j$ would have been a v-structure or we would have had a cycle; by the same argument, $V_j\in\adj{\C}{V_{K-2}}$, and we can proceed until we obtain $V_j\in\adj{\C}{V_{j+2}}$, which is a contradiction. (b) Assume instead that $V_{j-1}\in\adj{\C}{C}$ such that $V_j- C$ remains undirected. If $V_{j-2}\notin\adj{\C}{C}$, we obtain a contradiction as above; hence, assume that $V_{j-2}\in\adj{\C}{C}$. We can proceed with this until we obtain $V_1\in\adj{\C}{C}$. By transitivity, $\tau(V_1)<\tau (C)$ and we obtain $V_1\rightarrow C - A$ in $\C^\tau$. In order to obtain $A-C$ in $\C^{\tau,n}$, we must have $V_1\in\adj{\C}{A}$. By the same reasoning as above, the path then cannot be unshielded, and  we obtain a contradiction.

\emph{Rule 3:} If \ref{fig.meeksrules} (iii) is an induced subgraph of $\C^{\tau, n}$, then it is also an induced subgraph of $\C$, since the v-structure cannot be newly forced by Meek's 1st rule. However, \ref{fig.meeksrules} (iii) cannot be an induced subgraph of $\C$ since $\C$ is maximally informative.   

\emph{Rule 4:} Consider the induced subgraph of $\C$ over $\{A,B,C,D\}\subseteq\mathbf{V}$ with adjacencies as in Figure \ref{fig.meeksrules} (iv). By the same argument as above, only an undirected induced subgraph of $\C$ can lead to an induced subgraph like \ref{fig.meeksrules} (iv) in $\C^{\tau, n}$. Moreover, by the argument above, we know that Figure \ref{fig.meeksrules} (iv) does not occur as an induced subgraph of $\C^{\tau}$; hence, we consider the case where $\tau(A)=\tau(B)=\tau(C)=\tau(D)$ and this subgraph is undirected. 
The only way that $A\rightarrow B$ can be directed in $\C^{\tau,n}$ and not in $\C^\tau$ is if there is an undirected unshielded path $\langle V_1,\ldots ,V_K=A\rangle$ in $\C$ of length $K>1$ with $V_{K-1}\notin\adj{\C}{B}$. Assume that $\tau(V_1)<\tau (V_2)=\tau(V_3)=\ldots =\tau(A)$ such that $V_1\rightarrow V_2$ in $\C^{\tau}$ and $n\geq K-1$ applications of Meek's 1st rule results in $V_2\rightarrow\ldots \rightarrow A\rightarrow B$ in $\C^{\tau,n}$. If $V_{K-1}\notin\adj{\C}{C}$ then $A\rightarrow C$ will be forced by Meek's 1st rule. Hence, we assume that $V_{K-1}\in\adj{\C}{C}$. To obtain Figure \ref{fig.meeksrules} (iv) in $\C^{\tau,n}$ we require $C-B$ to be undirected; hence,  we can proceed the in a similar way as for Rule 2 and obtain a contradiction. 
\end{proof}

\subsection{Proof of Theorem \ref{theorem:cycles}}

\begin{proof}
Assume that $\C$ is the CPDAG of which $\G$ is constructed, and $\tau$ the tiered ordering. Let $\C^\tau$ denote the graph obtained by orienting edges in $\C$ according to $\tau$, and let $\C^{\tau, n}$ be the graph obtained by applying Meek's 1st rule to $\C^\tau$ $n$ times. By Lemma \ref{mainlemma}, there exists an $N$ such that for $n=N$ we have $\G=\C^{\tau, n}$; hence, we can without loss of generality assume $\C^{\tau, n}$ to be maximally informative. Since $\C$ does not contain any partially directed cycles, any partially directed cycle in $\G$ must be either (i) forced by $\tau$, or (ii) forced by Meek's 1st rule. Hence, any partially directed cycle in $\C^\tau$ or $\C^{\tau, n}$ must correspond to an undirected cycle in $\C$: Let $\langle V_1,\ldots V_K\rangle$ combined with $V_1 - V_K$ be an undirected cycle in $\C$. We will show that (i) the corresponding cycle in $\C^\tau$ cannot be partially directed, and (ii) the corresponding cycle in $\C^{\tau , n}$ cannot be partially directed.

(i) Without loss of generality, assume that $\tau (V_1)<\tau (V_2)$ such that the edge $V_1\rightarrow V_2$ is oriented in $\C^\tau$. If $\tau (V_1)<\tau (V_K)$ we will not obtain a partially directed cycle; therefore, assume that $\tau (V_K)\leq\tau (V_1)$. If for any $2\leq i\leq K-1:$ $\tau (V_i)>\tau (V_{i+1})$, again, it is no longer a partially directed cycle; therefore, assume $\tau (V_i)\leq\tau (V_{i+1})$ for all $2\leq i\leq K-1$. This then implies that $\tau (V_2)\leq\tau (V_K)\leq\tau (V_1)$. This is a contradiction to transitivity since we assumed $\tau (V_1)<\tau (V_2)$. We conclude that there cannot exist a partially directed cycle in $\C^\tau$.

(ii) By the above, there cannot be any partially directed cycles in $\C^\tau$; hence, if $\C^{\tau, n}$ contains a partially directed cycle, it must be forced through Meek's 1st rule; then $\tau (V_1)=\tau (V_2) =\ldots =\tau (V_K)$. Assume that there is an undirected unshielded path $\langle W_1,\ldots , W_m= V_1,V_2\rangle$ in $\C$, $m> 1$, with $\tau (W_1)<\tau(W_2)=\tau(W_3)=\ldots=\tau (V_1)$ such that $W_1\rightarrow W_2$ in $\C^\tau$, and assume that $n\geq m-1$ such that $W_1\rightarrow W_2\rightarrow\ldots\rightarrow W_{m-1}\rightarrow V_1\rightarrow V_2$ is in $\C^{\tau,n}$. If $W_{m-1}\notin\adj{\C}{V_K}$ the edge $V_1\rightarrow V_K$ follows from Meek's 1st rule and we no longer have a partially directed cycle; therefore, assume that $W_{m-1}\in\adj{\C}{V_K}$. Either (a) $W_{m-2}\notin\adj{\C}{V_K}$ or (b) $W_{m-2}\in\adj{\C}{V_K}$. (a) In this case $W_{m-1}\rightarrow V_K$ by Meek's 1st rule. If $W_{m-1}\notin\adj{\C}{V_{K-1}}$, then $V_K\rightarrow V_{K-1}$ and we no longer have a partially directed cycle; assume $W_{m-1}\in\adj{\C}{V_{K-1}}$. We can then proceed until we obtain $W_{m-1}\in\adj{\C}{V_2}$, which is a contradiction. (b) If $W_{m-3}\notin\adj{\C}{V_K}$, then $W_{m-2}\rightarrow V_K$ by Meek's 1st rule, and we obtain a contradiction as above. Hence, assume $W_{m-3}\in\adj{\C}{V_K}$. We can then proceed until we obtain $W_1\in\adj{\C}{V_K}$. By transitivity $\tau (W_1) <\tau (V_K)$ and the orientation $W_1\rightarrow V_K$ is forced by $\tau$. Assume that $W_1\rightarrow V_i$ for some $2<i\leq K$, then if $W_1\notin \adj{\C}{V_{i-1}}$, then $V_i\rightarrow V_{i-1}$ and we no longer have a partially directed cycle. Hence, assume that $W_1\in \adj{\C}{V_{i-1}}$ for all $2<i\leq K$. Then $W_1\in\adj{\C}{V_2}$ and for $m=2$ we have a contradiction. Assume $m>2$, then $W_1\in\adj{\C}{V_1}$ since otherwise we would have either a cycle or a v-structure $W_1\rightarrow V_2 \leftarrow V_1$, such that $V_1\rightarrow V_2$ would have been oriented in $\C$. Then $W_1\in\adj{\C}{W_{m-1}}$ since otherwise $W_{m-1}\rightarrow V_1$ would have been oriented in $\C$. We can proceed with this reasoning until we obtain $W_1\in\adj{\C}{W_3}$, which is a contradiction.
\end{proof}

\subsection{Proof of Corollary \ref{corollary:chain}}

\begin{proof}
In order to show that $\G$ is a chain graph it is sufficient to show that it does not contain any partially directed cycles, which is the case due to Theorem \ref{theorem:cycles}. Hence, we only need to show that the chain components are chordal: Assume that $\C$ is the CPDAG from which $\G$ is constructed. Assume $\pi$ is a chordless undirected cycle of length $\geq 4$ in $\G$; then $\pi$ must have been an undirected cycle in $\C$. Since $\C$ does not have any chordless undirected cycles, and since the procedure of orienting edges according to a tiered ordering or Meek's 1st rule does not delete edges or create partially directed cycles (c.f. Theorem \ref{theorem:cycles}), this is a contradiction.
\end{proof}

\subsection{Proof of Corollary \ref{corollary:possibly}}

The proof of Corollary \ref{corollary:possibly} follows directly from the following result:

\begin{corollary}
     Let $\G=(\V,\E)$ be a tiered MPDAG, and let $\pi=\langle V_1,\ldots ,V_K\rangle$ be a path in $\G$. Then $\pi$ is b-possibly causal from $V_1$ to $V_K$ if and only if it is possibly causal from $V_1$ to $V_K$.
\end{corollary}
 
\begin{proof}
``If'' Assume that $\pi$ is possibly causal from $V_1$ to $V_K$. Then there is no $V_i,V_j$ on $\pi$ with $i<j$ with $V_i\leftarrow V_j$ in $\G$, since otherwise $\langle V_i,\ldots ,V_j\rangle$ combined with $\langle V_j,V_i\rangle$ would constitute a partially directed cycle in $\G$, which would be a contradiction to Theorem \ref{theorem:cycles}.

``Only if'' Assume instead that $\pi$ is not possibly causal from $V_1$ to $V_K$. Then there is an edge  $V_i\leftarrow V_{i+1}$ for some $1\leq i\leq k$ on $\pi$. Then $\G$ contains $V_i,V_j$ on $\pi$ with $i<j$ with $V_i\leftarrow V_j$ and no path in $\G$ is then b-possibly causal from $V_1$ to $V_K$; in particular, $\pi$ is not b-possibly causal from $V_1$ to $V_K$.
\end{proof}

\subsection{Proof of Corollary \ref{corollary:ida}}

The proofs of the validity of the output of the local IDA-algorithm and the joint IDA-algorithm rely on the fact that in a CPDAG, no orientation of the undirected edges can lead to a new v-structure, or a cycle, that includes an edge that is already directed in the CPDAG \citep{meek1995}. It is straightforward to show that the same is true for tiered MPDAGs:

\begin{lemma}\label{suppl:lemma1}
    Let $\G=(\V,\E)$ be a tiered MPDAG, and let $\G_u$ and $\G_d$ be the undirected and the directed parts of $\G$ respectively. No orientation of the edges in $\G_u$ can create either (i) a v-structure in $\G$ that includes an edge in $\G_d$, or (ii) a cycle in $\G$ that includes an edge in $\G_d$.
\end{lemma}

\begin{proof}
    (i) By Lemma \ref{mainlemma} we know that $\G$ is maximal relative to Meek's 1st rule; this implies that no unshielded triple of the form $X_i\rightarrow X_j - X_k$ can occur in $\G$.

    (ii) Assume that we could orient the edges in $\G_u$ such that we would create a cycle in $\G$ including an edge from $\G_d$. This would require a cycle in $\G$ consisting of at least one directed part and at least one undirected part; however, this would constitute a partially directed cycle, which is a contradiction to Theorem \ref{theorem:cycles}.
\end{proof}

\begin{proof}[Proof of Corollary \ref{corollary:ida}]
    We will first consider the joint IDA, and we follow the proof of Theorem 5.1 in \citet{nandy2017estimating}: Let $\G_{u,1},\ldots ,\G_{u,n}$ denote the chain components of $\G_u$. Assume that only $\G_{u,1},\ldots ,\G_{u,l}$ contain a node from $\mathbf{X}$. By Lemma \ref{suppl:lemma1} we can orient each component $\G_{u,1},\ldots ,\G_{u,l}$ into DAGs independently of the rest of the graph and obtain all valid parent sets from these. The multiplicity statement follows directly from \citet{nandy2017estimating}.
    
    We will now turn to the local IDA and we will follow the proof of Lemma 3.1 in \citet{maathuis2009estimating}, which shows the following result: Let $X\in\V$ and let $\mathbf{S}\subset\nb{\G}{X}$, then $\G_{\mathbf{S}\rightarrow X}$ does not create new v-structures with $X$ as a collider if and only if there exists a DAG $\D$ in the (restricted) equivalence class represented by $\G$ for which $\pa{\D}{X}=\pa{\G}{X}\cup\mathbf{S}$. The "if" part is trivial, we show the "only if" part. As argued above, Lemma \ref{suppl:lemma1} allows us to consider each connected component of $\G_u$ separately. Assume that $X$ is in $\G_{u,i}$, we then need to show that we can orient $\G_{u,i}$ into a DAG without any new v-structures, where $\mathbf{S}$ is the parent set of $X$. In order to show that such an orientation exists, \citet{maathuis2009estimating} rely on two facts (1) the induced subgraph over $X\cup\mathbf{S}$ is complete, and (2) $\G_{u,i}$ is chordal. By Corollary \ref{corollary:chain} we know that (2) is satisfied. Since orienting edges from $\mathbf{S}$ into $X$ does not create any new v-structures, all nodes in $\mathbf{S}$ must be adjacent in $\G$; since $\mathbf{S}\subseteq\nb{\G}{X}$ it follows that the induced subgraph over $X\cup\mathbf{S}$ is complete. The rest follows from the proof of Lemma 3.1 in \citet{maathuis2009estimating}.
\end{proof}

\section{Proofs for section \ref{sec:character}}

\subsection{Proof of Theorem \ref{mainthm}}

\begin{proof}
We will make use of the following result: Let $\pi=\langle V_1, V_2, \ldots ,V_K\rangle$ be an unshielded path in $\C_u$, then $\pi$ is unshielded in $\C$ as well: If for any subpath $V_{k-1}-V_k-V_{k+1}$ of $\pi$ there were an edge $V_{k-1}\any V_{k+1}$ in $\C$ that was not in $\C_u$, then this edge would be directed; combined with $V_{k-1}-V_k-V_{k+1}$ this would then create a partially directed cycle, which cannot occur in $\C$ since it is a CPDAG. 

\emph{``Only if''}: $(i)$: Assume that (i) is violated. Let $\pi_1=\langle V_1, \ldots ,V_K\rangle$ be an unshielded path in $\C^{\tau_1}_u$ with $\pi_2=\langle V_1, \ldots ,V_K\rangle$ being the corresponding path in $\C^{\tau_2}_u$, and assume that the first cross-tier edge on $\pi_1$  is not the same as the first cross-tier edge on $\pi_2$. Additionally, assume that $\pi_1$ and $\pi_2$ are both earliest. 

Since $\pi_1$ and $\pi_2$ are unshielded and undirected, the corresponding paths in the underlying DAGs cannot contain colliders: They are either directed or they contain a subpath of the form $V_{k-1}\leftarrow V_k\rightarrow V_{k+1}$. In the latter case, either all cross-tier edges on $\pi_1$ will be on $\pi_1(V_1,V_k)$ or $\pi_1(V_k,V_K)$, or they will both contain cross-tier edges; similarly for $\pi_2$. It will then be sufficient to show that either  $\pi_1(V_1,V_k)\neq\pi_2(V_1,V_k)$ or $\pi_1(V_k,V_K)\neq\pi_2(V_k,V_K)$. Moreover, since we assume all background knowledge to be correct, the paths must agree on the direction. Hence, we can without loss of generality assume that the corresponding paths in the underlying DAGs are directed from $V_1$ to $V_K$. 

Assume that the first cross-tier edge on $\pi_1$ is $V_i\rightarrow V_{i+1}$ for $1\leq i\leq K$, while the first cross-tier edge on $\pi_2$ is $V_j\rightarrow V_{j+1}$ with $i < j\leq K$. Let $\pi'_1$ be the path in $\G_1$ corresponding to $\pi_1$, and let $\pi_2'$ be the corresponding path in $\G_2$. Since only Meek's 1st rule applies (c.f. Lemma \ref{mainlemma}), the subpath $\pi'_1(V_1,V_i)$ will remain undirected since no new arrowheads are oriented into this subpath. Assume for contradiction that for some $V_h$ with $1\leq h\leq i-1$ there were a node $W\in\adj{\C_u}{V_h}$ with $\tau_1 (W)< \tau_1 (V_h)$ such that $W\rightarrow V_h$ in $\C_u^{\tau_1}$. Then the path $\pi'=\langle W, V_h, V_{h+1},\ldots ,V_K\rangle$ in $\C_u^{\tau_1}$ would be earlier than $\pi_1$, and $\pi_1$ would contain the subpath $\langle V_h, V_{h+1},\ldots ,V_K\rangle$ of $\pi'$, which is a contradiction since we assumed $\pi_1$ to be earliest. The subpath $\pi'_1(V_i,V_K)$ will be directed: $V_i\rightarrow V_{i+1}$ is forced by $\tau_1$, and we will then be able to iteratively orient each node on $\langle V_{i+1}, \ldots ,V_K\rangle$ in the direction of $V_K$ according to Meek's 1st rule when constructing $\G_1$, c.f. Lemma \ref{mainlemma}. Analogously, the subpath of $\pi'_2(V_1,V_j)$ is undirected, while the subpath $\pi'_2(V_j,V_K)$ is directed in $\G_2$. Hence, we have that $\pi'_1(V_i,V_j)\neq \pi'_2(V_i,V_j)$: It then follows that $\G_1\neq\G_2$. 

$(ii)$: Assume that (ii) is violated. Let $V_i\any V_j$ be an edge for which $\C_u^{\tau_1}$ and $\C_u^{\tau_2}$ disagree on whether it is directed or not. Since $V_i\any V_j$ is only contained on shielded paths, it can only be oriented by background knowledge c.f. Lemma \ref{mainlemma}, since Meek's 1st rule does not apply. It follows that $\G_1\neq\G_2$.

\emph{``If''}: Since $\G_1$ and $\G_2$ are constructed from the same CPDAG, they will agree on every edge that is directed in $\C$; hence, we will consider $\C_u$. Assume that (i) and (ii) are both satisfied. By (ii) we know that $\G_1$ and $\G_2$ will agree on the orientation of any fully shielded edge, so we need to show that they will also agree on the orientation of any edge that is not fully shielded; we will consider the unshielded paths. 

Let $\pi_1=\langle V_1, \ldots ,V_K\rangle$ be an unshielded path in $\C_u^{\tau_1}$ and let  $\pi_2=\langle V_1, \ldots ,V_K\rangle$ be the corresponding path in $\C_u^{\tau_2}$. Assume that $\D_1\in [\C]$ is a DAG giving rise to $\tau_1$ and $\D_2\in [\C]$ is a DAG giving rise to $\tau_2$. By the same argument as above, we may assume that either (a) the corresponding paths in $\D_1$ and $\D_2$ are directed from $V_1$ to $V_K$, or (b) the corresponding path in $\D_1$ contains $V_{k-1}\leftarrow V_k\rightarrow V_{k+1}$ for some $2\leq k\leq {K-1}$; i.e. the subpaths  will be directed from $V_k$ to $V_1$ and from $V_k$ to $V_K$, and the corresponding path in $\D_2$ contains $V_{l-1}\leftarrow V_l\rightarrow V_{l+1}$ for some $2\leq l\leq {K-1}$; i.e. the subpaths will be directed from $V_l$ to $V_1$ and from $V_l$ to $V_K$, or (c) the corresponding path in one DAG is directed from $V_1$ to $V_K$, and the corresponding path in the other DAG contains a subpath $V_{k-1}\leftarrow V_k\rightarrow V_{k+1}$ for some $2\leq k\leq {K-1}$. Since (b) is the most general case, we will only consider this; (a) and (c) can be verified in a similar way.

Either $\pi_1(V_1,V_k)$ and $\pi_2(V_1,V_l)$ will have a cross-tier edge, $\pi_1(V_k,V_K)$ and $\pi_2(V_l,V_K)$ will have a cross-tier edge, or they will all have a cross-tier edge. We consider the most general case where they all have a cross-tier edge, and assume that the first cross-tier edge on $\pi_1(V_k,V_K)$ and $\pi_2(V_l,V_K)$ is $V_{i}\rightarrow V_{i+1}$ and that the first cross-tier edge on $\pi_1(V_1,V_k)$ and $\pi_2(V_1,V_l)$ is $V_{j}\rightarrow V_{j-1}$. Let $\pi'_1$ be the path in $\G_1$ corresponding to $\pi_1$, and let $\pi_2'$ be the corresponding path in $\G_2$. By similar arguments as above, it then follows that $\pi_1'(V_j, V_i)=\pi_2'(V_j,V_i)$ will remain undirected, $\pi_1'(V_1, V_j)=\pi_2'(V_1,V_j)$ will be directed from $V_j$ to $V_1$, and $\pi_1'(V_i, V_K)=\pi_2'(V_i,V_K)$ will be directed from $V_i$ to $V_K$. The case where $\pi_1$ and $\pi_2$ only have a single cross-tier edge is special case of this. Hence, $\pi_1'=\pi_2'$.
\end{proof}

\subsection{Proof of Corollary \ref{corollary:informative}}

\begin{proof}
Let $\G_1$ be the MPDAG obtained from $\C$ relative to $\tau_1$, and let $\G_2$ be the MPDAG obtained from $\C$ relative to $\tau_2$. Assume that (i) and (ii) are satisfied. If $\C^{\tau_1}_u$ does not have any additional oriented edges, then $\G_1=\G_2$ by Theorem \ref{mainthm}.

Assume that (i), (ii), and (iii) are satisfied. Let $\pi_1=\langle V_1,\ldots V_K\rangle$ be an earliest unshielded path in $\C^{\tau_1}_u$ and let $V_i\rightarrow V_{i+1}$ be the first cross-tier edge on $\pi_1$. Let $\pi_1'$ be the corresponding path in $\G_1$. Then $\pi'_1(V_1,V_i)$ will be undirected and $\pi'_1(V_1,V_i)$ will be directed, by similar arguments as in the proof of Theorem \ref{mainthm}. Let $\pi_2=\langle V_1,\ldots V_K\rangle$ be the path in $\C^{\tau_2}_u$ corresponding to $\pi_1$ and assume that $V_i-V_{i+1}$ is not a cross-tier edge in $\C^{\tau_2}_u$. Let $\pi_2'$ be the corresponding path in $\G_2$. Either $\pi_2$ will have at least one cross-tier edge, or it will have no cross-tier edges. If $\pi_2$ has no cross-tier edges, then $\pi_2'$ will be undirected: Since $\pi_1'$ will be directed from $V_i$ to $V_K$, $\G_1$ will be contained in $\G_2$. Assume instead that $\pi_2$ has at least one cross-tier edge and that the first cross-tier edge is $V_j\rightarrow V_{j+1}$. Then by (i) this is also a cross-tier edge on $\pi_1$. Since $V_j\rightarrow V_{j+1}$ is not the first cross-tier edge on $\pi_1$ it follows that $i\leq j$; since $V_i\rightarrow V_{i+1}$ is not a cross-tier edge on $\pi_2$ we conclude that $i<j$. By similar arguments as in the proof of Theorem \ref{mainthm} we then know that $\pi'_1(V_1,V_i)=\pi'_2(V_1,V_i)$ are undirected, $\pi'_1(V_j,V_K)=\pi'_2(V_j,V_K)$ are directed, and $\pi'_1(V_i,V_j)\neq\pi'_2(V_i,V_j)$ since $\pi'_1(V_i,V_j)$ is directed and $\pi'_2(V_i,V_j)$ is undirected. Then $\G_1$ will be contained in $\G_2$ and $\tau_1$ will be more informative than $\tau_2$.

Assume that (i), (ii) and (iv) are satisfied. Following the proof of Theorem \ref{mainthm}, the fully shielded edges can only be oriented by background knowledge and  $\G_1$ will be contained in $\G_2$, and $\tau_1$ will be more informative than $\tau_2$.

Assume that (i), (ii), (iii) and (iv) are all satisfied. Then by the same arguments as above, $\G_1$ will be contained in $\G_2$, and $\tau_1$ will be more informative than $\tau_2$.
\end{proof}

\end{document}